\documentclass{article} 
\usepackage{amsfonts}
\usepackage{url}
\usepackage{multirow}
\usepackage{algorithm,algorithmic,amsmath,amsthm,amsfonts,dsfont}
\usepackage{verbatim,graphicx}
\usepackage{mathtools}
\usepackage[usenames,dvipsnames,svgnames,table]{xcolor}
\newcommand{\BlackBox}{\rule{1.5ex}{1.5ex}}  
 
\newtheorem*{assumption*}{Assumption}
 
\newtheorem{thm}{Theorem}

\newtheorem{lemma}{Lemma} 
\newtheorem{proposition}{Proposition}

\def\ones{\mathds{1}}

\def\bbE{\mathbb{E}}

\def\bbP{\mathbb{P}}
\def\bbR{\mathbb{R}}
\def\bbN{\mathbb{N}}

\def\cG{\mathcal{G}}

\def\cE{\mathcal{E}}

\def\cP{\mathcal{P}}

\def\cR{\mathcal{R}}

\def\cR{\mathcal{R}}

\newcommand{\defeq}{\mbox{$\;\stackrel{\mbox{\tiny\rm def}}{=}\;$}}

\newcommand{\leqa}{\mbox{$\;\stackrel{\mbox{\tiny\rm (a)}}{\leq}\;$}}
\newcommand{\leqb}{\mbox{$\;\stackrel{\mbox{\tiny\rm (b)}}{\leq}\;$}}
\newcommand{\leqc}{\mbox{$\;\stackrel{\mbox{\tiny\rm (c)}}{\leq}\;$}}
\newcommand{\leqd}{\mbox{$\;\stackrel{\mbox{\tiny\rm (d)}}{\leq}\;$}}

\DeclareMathOperator{\sgn}{sgn}

\DeclarePairedDelimiter\abs{\lvert}{\rvert}%
\DeclarePairedDelimiter\norm{\lVert}{\rVert}%

\makeatletter
\let\oldabs\abs
\def\abs{\@ifstar{\oldabs}{\oldabs*}}
\let\oldnorm\norm
\def\norm{\@ifstar{\oldnorm}{\oldnorm*}}
\makeatother

\usepackage{times,nips15submit_e}
\usepackage{todonotes}
\usepackage{graphicx,caption,subcaption}
\usepackage{algorithm,algorithmic}

\usepackage[square,sort,comma,numbers]{natbib}

\newcommand{\Zhat}{\hat{Z}}
\newcommand{\cL}{\mathcal{L}}
\newcommand{\siom}{|\Omega|}
\newcommand{\Deltaom}{\Delta_{\Omega}}

\newcommand{\Xom}{X_{\Omega}}
\newcommand{\Zhatom}{\hat{Z}_{\Omega}}
\newcommand{\ghat}{\hat{g}}
\newcommand{\Mhat}{\hat{M}}
\newcommand{\rank}{rank}
\newcommand{\Mhatij}{\hat{M}_{i,j}}
\newcommand{\Zhatij}{\hat{Z}_{i,j}}
\newcommand{\Bhat}{\hat{B}}

\newcommand{\Mij}{M_{i,j}}
\newcommand{\Mstar}{M^{\star}}
\newcommand{\Zstar}{Z^{\star}}
\newcommand{\Mstarij}{M^{\star}_{i,j}}
\newcommand{\Zstarij}{Z^{\star}_{i,j}}
\newcommand{\Xstar}{X^\star}
\newcommand{\Xij}{X_{i,j}}
\newcommand{\gstar}{g^{\star}}
\newcommand{\mmc}{\text{MMC}}
\newcommand{\LPAV}{LPAV}
\newcommand{\MSE}{MSE}
\newcommand{\cN}{\mathcal{N}}

\newcommand{\limco}{\text{MMC}}
\nipsfinalcopy
\title{Matrix Completion Under Monotonic Single Index Models }
\author{
Ravi Ganti \\
Wisconsin Institutes for Discovery\\
UW-Madison\\
\texttt{gantimahapat@wisc.edu} \\
\And
Laura Balzano \\
Electrical Engineering and Computer Sciences \\
University of Michigan Ann Arbor \\
\texttt{girasole@umich.edu} \\
\AND
Rebecca Willett\\
Department of Electrical and Computer Engineering\\
UW-Madison\\
\texttt{rmwillett@wisc.edu} 
}
\begin{document}
\maketitle
\begin{abstract}
Most recent results in matrix completion assume that the matrix under consideration is low-rank or that the columns are in a union of low-rank subspaces. In real-world settings, however, the linear structure underlying these models is distorted by a (typically unknown) nonlinear transformation. This paper addresses the challenge of matrix completion in the face of such nonlinearities. Given a few observations of a matrix that are obtained by applying a Lipschitz, monotonic function to a low rank matrix, our task is to estimate the remaining unobserved entries. We propose a novel matrix completion method that alternates between low-rank matrix estimation and monotonic function estimation to estimate the missing matrix elements. Mean squared error bounds provide insight into how well the matrix can be estimated based on the size,  rank of the matrix and properties of the nonlinear transformation. Empirical results on synthetic and real-world datasets demonstrate the competitiveness of the proposed approach.  
\end{abstract}
\section{Introduction}
In matrix completion, one has access to a matrix with only a few observed entries, and the task is to estimate the entire matrix using the observed entries. This problem has a plethora of applications such as collaborative filtering, recommender systems~\cite{melville2010recommender} and sensor networks~\cite{cucuringu2012graph}. Matrix completion has been well studied in machine learning, and we now know how to recover certain matrices given a few observed entries of the matrix~\cite{recht2011simpler,candes2009exact} when it is assumed to be low rank. Typical work in matrix completion assumes that the matrix to be recovered is incoherent, low rank, and entries are sampled uniformly at random~\cite{candes2010matrix,negahban2012restricted,candes2009exact,recht2011simpler,keshavan2010matrix,gross2011recovering}. While recent work has focused on relaxing the incoherence and sampling conditions under which matrix completion succeeds, there has been little work for matrix completion when the underlying matrix is of high rank. More specifically, we shall assume that the matrix that we need to complete is obtained by applying some unknown, non-linear function to each element of an unknown low-rank matrix. Because of the application of a non-linear transformation, the resulting ratings matrix tends to have a large rank. To understand the effect of the application of non-linear transformation on a low-rank matrix, we shall consider the following simple experiment:  Given an $n\times m$ matrix $X$, let $X=\sum_{i=1}^m \sigma_i u_i v_i^\top$ be its SVD. The rank of the matrix $X$ is the number of non-zero singular values. Given an $\epsilon\in (0,1)$, define the effective rank of $X$ as follows:
\begin{equation}
\label{eqn:eff_rank}
r_{\epsilon}(X)=\min\left\{k\in\bbN:\sqrt{\frac{\sum_{j=k+1} ^m \sigma_j^2}{\sum_{j=1}^m  \sigma_j^2}}\leq \epsilon\right\}.
\end{equation}
The effective rank of $X$ tells us the rank $k$ of the lowest rank approximator $\hat{X}$ that satisfies
\begin{equation}
\frac{||\hat{X}-X||_F}{||X||_F}\leq \epsilon.
\end{equation}
\begin{figure}[t]
  \centering
  \includegraphics[scale=0.25]{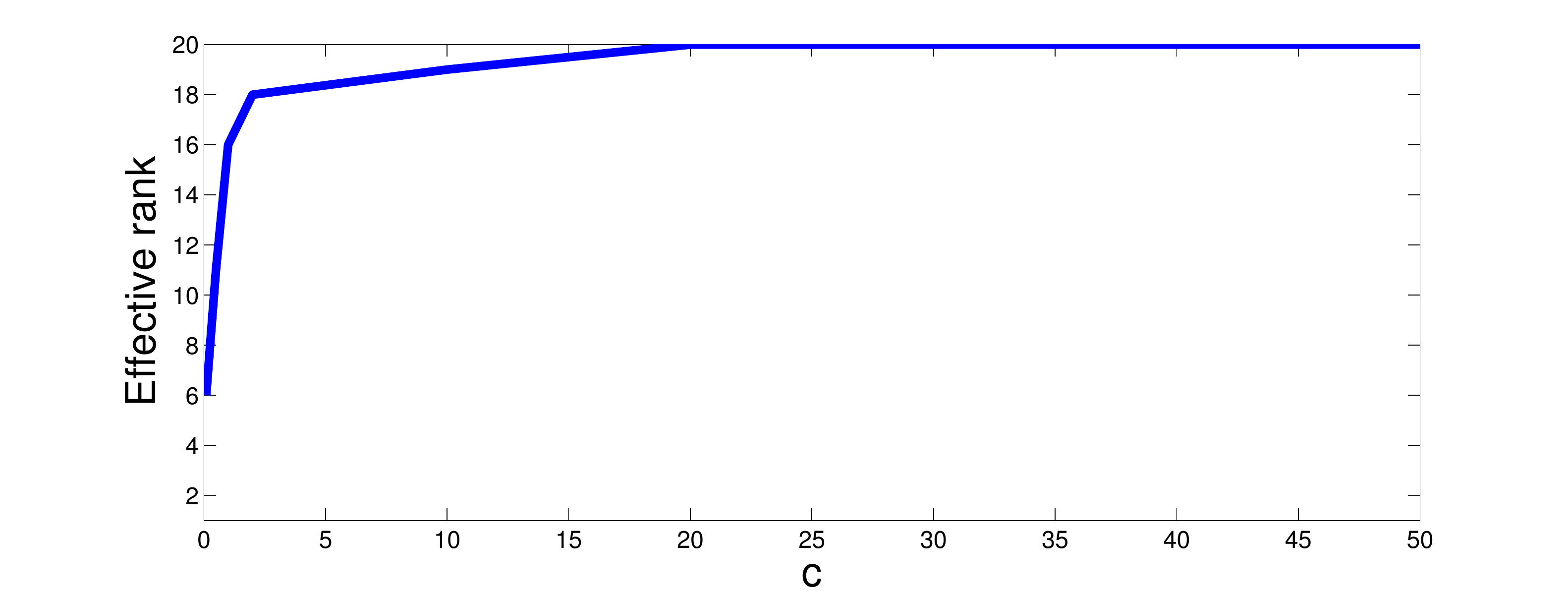}
\caption{The plot shows the $r_{0.01}(X)$ defined in equation~\eqref{eqn:eff_rank} obtained by applying a non-linear function $\gstar$ to each element of $Z$, where $\gstar(z)=\frac{1}{1+\exp(-cz)}$. $Z$ is a $30\times 20$ matrix of rank $5$.\label{fig:hrank}}
\end{figure}
In figure~\eqref{fig:hrank}, we show the effect of applying a non-linear monotonic function $\gstar(z)=\frac{1}{1+\exp(-cz)}$ to the elements of a low-rank matrix $Z$. As $c$ increases both the rank of $X$ and its effective rank $r_{\epsilon}(X)$ grow rapidly with $c$, rendering traditional matrix completion methods ineffective even in the presence of mild nonlinearities.
\subsection{Our Model and contributions}
\label{sec:model}
In this paper we consider the high-rank matrix completion problem where the data generating process is as follows:
There is some unknown matrix $\Zstar\in \bbR^{n\times m}$ with $m\leq n$ and of rank $r\ll m$. A non-linear, monotonic, $L$- Lipschitz function $\gstar$ is applied to each element of the matrix $\Zstar$ to get another matrix $\Mstar$. A noisy version of $\Mstar$, which we call $X$, is observed on a subset of indices denoted by  $\Omega \subset [n]\times [m]$. 
\begin{align}
\Mstarij&=\gstar(\Zstarij),~\forall i\in [n], j\in [m] \label{eqn:model2}\\
X_{\Omega}&=(\Mstar+N)_{\Omega}\label{eqn:model1}
\end{align}
 The function $\gstar$ is called the transfer function. We shall assume that $\bbE[N]=0$, and the entries of $N$ are i.i.d. We shall also assume that the index set $\Omega$ is generated uniformly at random with replacement from the set $[n]\times [m]$~\footnote{By $[n]$ we denote the set $\{1,2\ldots,n\}$}. Our task is to reliably estimate the entire matrix $\Mstar$ given observations of $X$ on $\Omega$. We shall call the above model as Monotonic Matrix Completion (MMC). To illustrate our framework we shall consider the following two simple examples. In recommender systems users are required to provide discrete ratings of various objects. For example, in the Netflix problem users are required to  rate movies on a scale of $1-5$~\footnote{This is typical of many other recommender engines such as Pandora.com, Last.fm and Amazon.com.}. These discrete scores can be thought of as obtained by applying a rounding function to some ideal real valued score matrix given by the users. This real-valued score matrix may be well modeled by a low-rank matrix, but the application of the rounding function~\footnote{Technically the rounding function is not a Lipschitz function but can be well approximated by a Lipschitz function.} increases the rank of the original low-rank matrix. Another important example is that of completion of Gaussian kernel matrices. Gaussian kernel matrices are used in kernel based learning methods. The Gaussian kernel matrix of a set of $n$ points is an $n\times n$ matrix obtained by applying the Gaussian function on an underlying Euclidean distance matrix. The Euclidean distance matrix is a low-rank matrix~\cite{dattorro2010convex}. However, in many cases one cannot measure all pair-wise distances between objects, resulting in an incomplete Euclidean distance matrix and hence an incomplete kernel matrix. Completing the kernel matrix can then be viewed as completing a matrix of large rank.

In this paper we study this matrix completion problem and provide algorithms with provable error guarantees. Our contributions are as follows:
\begin{enumerate}
\item In Section~\eqref{sec:alg} we propose an optimization formulation to estimate matrices in the above described context. In order to do this we introduce two formulations, one using a squared loss, which we call           $\mmc$ - LS, and another using a calibrated loss function, which we call as $\mmc-c$. For both these formulations we minimize w.r.t. $\Mstar$ and $\gstar$. This calibrated loss function has the property that the minimizer of the calibrated loss satisfies equation~\eqref{eqn:model2}. 
\item We propose alternating minimization algorithms to solve our optimization problem. Our proposed algorithms, called $\mmc-c$ and $\mmc$-LS, alternate between solving a quadratic program to estimate $\gstar$ and performing projected gradient descent updates to estimate the matrix $\Zstar$. $\mmc$ outputs the matrix $\Mhat$ where $\Mhat_{i,j}=\ghat(\Zhat_{i,j})$.
\item In Section~\eqref{sec:analysis} we analyze the mean squared error (MSE) of the matrix $\Mhat$ returned by one step of the $\limco-c$ algorithm. The upper bound on the MSE of the matrix $\Mhat$ output by $\limco$ depends only on the rank $r$ of the matrix $\Zstar$ and not on the rank of matrix $\Mstar$. This property makes our analysis useful because the matrix $\Mstar$ could be potentially high rank and our results imply reliable estimation of a high rank matrix with error guarantees that depend on the rank of the matrix $\Zstar$.   
\item We compare our proposed algorithms to state-of-art implementations of low rank matrix completion on both synthetic and real datasets (Section~\ref{sec:expts}). We develop an ADMM algorithm to solve the quadratic program that is used to estimate $\gstar$. We believe that the proposed ADMM algorithm is useful in its own right and can be used in general isotonic regression problems elsewhere~\cite{kalai2009isotron,kakade2011efficient}.
\end{enumerate}
\section{Related work}
\label{sec:rw}
Classical matrix completion with and without noise has been investigated by several authors~\cite{candes2010matrix,negahban2012restricted,candes2009exact,recht2011simpler,keshavan2010matrix,gross2011recovering}. The recovery techniques proposed in these papers solve a convex optimization problem that minimizes the nuclear norm of the matrix subject to convex constraints. Progress has also been made on designing efficient algorithms to solve the ensuing convex optimization problem~\cite{vandereycken2013low,tan2014riemannian,wang2014rank,wen2012solving}. Recovery techniques based on nuclear norm minimization guarantee matrix recovery under the condition that  a) the matrix is low rank, b)  the matrix is incoherent or not very spiky, and  c) the entries are observed uniformly at random.  
Literature on high rank matrix completion is relatively sparse. When columns or rows of the matrix belong to a union of subspaces, then the matrix tends to be of high rank. For such high rank matrix completion problems algorithms have been proposed that exploit the fact that multiple low-rank subspaces can be learned by clustering the columns or rows and learning subspaces from each of the clusters. While Eriksson et al.~\cite{eriksson2012high} suggested looking at the neighbourhood of each incomplete point for completion, ~\cite{yang2015sparse} used a combination of spectral clustering techniques as done in~\cite{soltanolkotabi2012geometric,elhamifar2013sparse} along with learning sparse representations via convex optimization to estimate the incomplete matrix. Singh et al.~\cite{singh2012completion} consider a certain specific class of high-rank matrices that are obtained from ultra-metrics.  In~\cite{koyejo2013retargeted} the authors consider a model similar to ours, but instead of learning a single monotonic function, they learn multiple monotonic functions, one for each row of the matrix. However, unlike in this paper, their focus is on a ranking problem and their proposed algorithms lack theoretical guarantees.

 Davenport et al~\cite{davenport20141} studied the one-bit matrix completion problem. Their model is a special case of the matrix completion model considered in this paper. In the one-bit matrix completion problem we assume that $\gstar$ is known and is the CDF of an appropriate probability distribution, and the matrix $X$ is a boolean matrix where each entry takes the value 1 with probability $M_{i,j}$, and $0$ with probability $1-M_{i,j}$. Since $\gstar$ is known, the focus in one-bit matrix completion problems is accurate estimation of $\Zstar$.


To the best of our knowledge the MMC model considered in this paper has not been investigated before. The MMC model is inspired by the single-index model (SIM) that has been studied both in statistics~\cite{kakade2011efficient,kalai2009isotron} and econometrics for regression problems~\cite{ichimura1993semiparametric,horowitz1996direct}. Our MMC model can be thought of as an extension of SIM to matrix completion problems.

\section{Algorithms for matrix completion}
\label{sec:alg}
Our goal is to estimate $\gstar$ and $\Zstar$ from the model in equations~(\ref{eqn:model2}-~\ref{eqn:model1}). We approach this problem via mathematical optimization.  Before we discuss our algorithms, we mention in brief an algorithm for the problem of learning Lipschitz, monotonic functions in $1$- dimension. This algorithm will be used for learning the link function in MMC.
\paragraph{The LPAV algorithm: }
Suppose we are given data $(p_1,y_1),\ldots (p_n,y_n)$, where $p_1\leq p_2\ldots\leq p_n$, and $y_1,\ldots,y_n$ are real numbers. Let $ \mathcal{G}\defeq\{g:\bbR \rightarrow \bbR, ~g~\text{ is L-Lipschitz and monotonic}\}$. The LPAV~\footnote{LPAV stands for Lipschitz Pool Adjacent Violator} algorithm introduced in~\cite{kakade2011efficient} outputs the best function $\hat{g}$ in $\cG$ that minimizes $\sum_{i=1}^n (g(p_i)-y_i)^2$. In order to do this, the LPAV first solves the following optimization problem: 
\begin{equation}
\label{opt:lpav}
  \hat{z}=\arg\min_{z\in\bbR^n} ~ \|z-y\|_2^2 \quad \textbf{s.t.} ~\ 0 \leq z_j-z_i\leq L(p_j-p_i) ~\text{if~} p_i \leq p_j
  \end{equation}
where, $\hat{g}(p_i)\defeq\hat{z}_i$. This gives us the value of $\hat{g}$ on a discrete set of points $p_1,\ldots,p_n$. To get $\hat{g}$ everywhere else on the real line, we simply perform linear interpolation as follows: 
\begin{equation}
   \label{eqn:interpolate}
    \hat{g}(\zeta)=
    \begin{cases}
     \hat{z}_{1}, & \text{if}\ \zeta \leq p_{1} \\
      \hat{z}_{n}, & \text{if}\ \zeta \geq p_{n} \\
      \mu \hat{z}_{i} + (1 - \mu) \hat{z}_{i+1} & \text{if}\ \zeta = \mu p_{i} + (1-\mu) p_{i+1}
    \end{cases}
  \end{equation} 

\subsection{Squared loss minimization}
\label{sec:squared_loss_approach}
 A natural approach to the monotonic matrix completion problem is to learn $\gstar,\Zstar$ via squared loss minimization. In order to do this we need to solve the following optimization problem:
 \begin{equation}
\begin{aligned}
\label{eqn:least_squares}
\min_{g,Z}& \sum_{\Omega} (g(Z_{i,j})-X_{i,j})^2\\
&g:\bbR\rightarrow \bbR~\text{is L-Lipschitz and monotonic}\\
&\rank(Z)\leq r.
\end{aligned} 
\end{equation}
 The problem is a non-convex optimization problem individually in parameters $g,Z$. A reasonable approach to solve this optimization problem would be to perform optimization w.r.t. each variable while keeping the other variable fixed. For instance, in iteration $t$, while estimating $Z$ one would keep $g$ fixed, to say $g^{t-1}$, and then perform projected gradient descent w.r.t. $Z$. This leads to the following updates for $Z$:
\begin{align}
\label{eqn:possible_update}
Z_{i,j}^{t}&\leftarrow Z_{i,j}^{t-1}-\eta (g^{t-1}(Z^{t-1}_{i,j})-X_{i,j}) (g^{t-1})'(Z_{i,j}^{t-1})~,\forall (i,j)\in \Omega\\
Z^{t}& \leftarrow P_r(Z^{t})
\end{align}
where $\eta>0$ is a step-size used in our projected gradient descent procedure, and $P_r$ is projection on the rank $r$ cone. The above update involves both the function $g^{t-1}$ and its derivative $(g^{t-1})'$. Since our link function is monotonic, one can use the LPAV algorithm to estimate this link function $g^{t-1}$. Furthermore since LPAV estimates $g^{t-1}$ as a piece-wise linear function, the function has a sub-differential everywhere and the sub-differential $(g^{t-1})'$ can be obtained very cheaply. Hence, the projected gradient update shown in equation~\eqref{eqn:possible_update} along with the LPAV algorithm can be iteratively used to learn estimates for $\Zstar$ and $\gstar$. We shall call this algorithm as $\mmc-$LS. Incorrect estimation of $g^{t-1}$ will also lead to incorrect estimation of the derivative $(g^{t-1})'$. Hence, we would expect $\mmc-$ LS to be less accurate than a learning algorithm that does not have to estimate $(g^{t-1})'$. We next outline an approach that provides a principled way to derive updates for $Z^t$ and $g^t$ that does not require us to estimate derivatives of the transfer function, as in $\mmc-$ LS.

\subsection{Minimization of a calibrated loss function and the MMC algorithm.} Let $\Phi:\bbR\rightarrow \bbR$ be a differentiable function that satisfies $\Phi'=\gstar$. Furthermore, since $\gstar$ is a monotonic function, $\Phi$ will be a convex loss function. Now, suppose $\gstar$ (and hence $\Phi$) is known. Consider the following function of $Z$
\begin{equation}
\label{eqn:calibrated_loss}
\cL(Z;\Phi,\Omega)=\bbE_{X} \left(\sum_{(i,j)\in \Omega}\Phi(Z_{i,j})-X_{i,j}Z_{i,j}\right).
\end{equation}
The above loss function is convex in $Z$, since $\Phi$ is convex.
Differentiating the expression on the R.H.S. of Equation~\ref{eqn:calibrated_loss} w.r.t. $Z$, and setting it to 0, we get
\begin{equation}
\label{eqn:derivative}
\sum_{(i,j)\in\Omega} \gstar(Z_{i,j})-\bbE X_{i,j}=0.
\end{equation} 
The $\mmc$ model shown in Equation~\eqref{eqn:model2} satisfies Equation~\eqref{eqn:derivative} and is therefore a minimizer of the loss function  $\cL(Z;\Phi,\Omega)$. Hence, the loss function~\eqref{eqn:calibrated_loss} is ``calibrated'' for the $\mmc$ model that we are interested in. The idea of using calibrated loss functions was first introduced for learning single index models~\cite{agarwal2014least}. When the transfer function is identity, $\Phi$ is a quadratic function and we get the squared loss approach that we discussed in section~\eqref{sec:squared_loss_approach}. 

The above discussion assumes that $\gstar$ is known. However in the $\mmc$ model this is not the case. To get around this problem, we consider the following optimization problem
\begin{equation}
\label{eqn:calibrated_problem}
\min_{\Phi,Z} \cL(\Phi,Z;\Omega)=\min_{\Phi,Z} \bbE_X\sum_{(i,j)\in\Omega} \Phi(Z_{i,j})- X_{i,j}Z_{i,j}
\end{equation} 
where $\Phi:\bbR\rightarrow \bbR$ is a convex function, with $\Phi'=g$ and $Z\in \bbR^{m\times n}$ is a low-rank matrix. Since, we know that $\gstar$ is a Lipschitz, monotonic function, we shall solve a constrained optimization problem that enforces Lipschitz constraints on $g$ and low rank constraints on $Z$.  We consider the sample version of the optimization problem shown in equation~\eqref{eqn:calibrated_problem}.
\begin{equation}
\label{eqn:calibrated_problem_sample}
\min_{\substack{\Phi\\ \rank(Z)\leq r}} \cL(\Phi,Z;\Omega)=\min_{\Phi,Z} \sum_{(i,j)\in\Omega} \Phi(Z_{i,j})- X_{i,j}Z_{i,j}
\end{equation} 

 The pseudo-code of our algorithm $\mmc$ that solves the above optimization problem~\eqref{eqn:calibrated_problem_sample} is shown in algorithm~\eqref{alg:mmc}. $\mmc$ optimizes for $\Phi$ and $Z$ alternatively, where we fix one variable and update another. 

 At the start of iteration $t$, we have at our disposal iterates $\hat{g}^{t-1}$, and $Z^{t-1}$. To update our estimate of $Z$, we perform gradient descent with fixed $\Phi$ such that $\Phi'=\hat{g}^{t-1}$. Notice that the objective in equation~\eqref{eqn:calibrated_problem_sample} is convex w.r.t. $Z$. This is in contrast to the least squares objective where the objective in equation~\eqref{eqn:least_squares} is non-convex w.r.t. $Z$.  The gradient of $\cL(Z;\Phi)$ w.r.t. $Z$ is
   \begin{equation}
   \nabla_{Z_{i,j}} \cL(Z;\Phi)= \sum_{(i,j)\in \Omega}\hat{g}^{t-1}(\hat{Z}_{i,j}^{t-1})-X_{i,j}. 
   \end{equation}
   Gradient descent updates on $\hat{Z}^{t-1}$ using the above gradient calculation leads to an update of the form
   \begin{align}
   \hat{Z}_{i,j}^{t}&\leftarrow \hat{Z}_{i,j}^{t-1}-\eta (\hat{g}^{t-1}(\hat{Z}_{i,j}^{t-1})-X_{i,j})\ones_{(i,j)\in\Omega}\nonumber\\
   \hat{Z}^t&\leftarrow \cP_r(\hat{Z}^t)\label{eqn:project}
   \end{align}
   Equation~\eqref{eqn:project} projects matrix $\hat{Z}^t$ onto a cone of matrices of rank $r$. This entails performing SVD on $\hat{Z}^t$ and retaining the top $r$ singular vectors and singular values while  discarding the rest. This is done in steps 4, 5 of Algorithm~\eqref{alg:mmc}.
   As can be seen from the above equation we \emph{do not} need to estimate derivative of $\hat{g}^{t-1}$. This, along with the convexity of the optimization problem in Equation~\eqref{eqn:calibrated_problem_sample} w.r.t. $Z$ for a given $\Phi$ are two of the key advantages of using a calibrated loss function over the previously proposed squared loss minimization formulation. 

   \textbf{Optimization over $\Phi$.} In round $t$ of algorithm~\eqref{alg:mmc}, we have $\hat{Z}^t$ after performing steps 4, 5. Differentiating the objective function in equation~\eqref{eqn:calibrated_problem_sample} w.r.t. $Z$, we get that the optimal $\Phi$ function should satisfy 
\begin{equation}
\label{eqn:gupdate}
\sum_{(i,j)\in\Omega} \hat{g}^t(\hat{Z}_{i,j}^{t})-X_{i,j}=0, 
\end{equation}
where $\Phi'=\hat{g}^t$.
 This provides us with a strategy to calculate $\hat{g}^t$. Let, $\hat{X}_{i,j}\defeq \hat{g}^t(\hat{Z}_{i,j}^{t})$. Then solving the optimization problem in equation~\eqref{eqn:gupdate} is equivalent to solving the following optimization problem.
\begin{equation}
\label{eqn:lpav}
\begin{aligned}
    \underset{\hat{X}}{\text{min}}&~ \sum_{(i,j)\in \Omega} (\hat{X}_{i,j}-X_{i,j})^2\\
  \text{subject to:}&~0\leq -\hat{X}_{i,j}+ \hat{X}_{k,l}\leq L(\hat{Z}_{k,l}^{t}-\hat{Z}_{i,j}^{t})~~\text{if}~ \hat{Z}_{i,j}^{t}\leq \hat{Z}_{k,l}^{t},~~ (i,j)\in \Omega, (k,l)\in \Omega\\
   \end{aligned}
   \end{equation} 
   where $L$ is the Lipschitz constant of $\gstar$. We shall assume that $L$ is known and does not need to be estimated. The gradient, w.r.t. $\hat{X}$, of the objective function, in equation~\eqref{eqn:lpav}, when set to zero is the same as Equation~\eqref{eqn:gupdate}. The constraints enforce monotonicity of $\hat{g}^t$ and the Lipschitz property of $\hat{g}^t$. The above optimization routine is exactly the LPAV algorithm. The solution $\hat{X}$ obtained from solving the LPAV problem can be used to define $\hat{g}^t$ on $X_{\Omega}$. These two steps are repeated for $T$ iterations. After $T$ iterations we have $\hat{g}^T$ defined on $\hat{Z}^T_{\Omega}$. In order to define $\hat{g}^T$ everywhere else on the real line we perform linear interpolation as shown in equation~\eqref{eqn:interpolate}.
\begin{algorithm}[H]
  \caption{Monotonic Matrix Completion (MMC)\label{alg:mmc}}
  \begin{algorithmic}[1]
  \REQUIRE  Parameters $\eta>0,T>0,r$, Data:$X_{\Omega},\Omega$
  \ENSURE $ \hat{M}=\hat{g}^T(\hat{Z}^T)$
  \STATE Initialize $\hat{Z}^0=\frac{mn}{\siom}X_{\Omega}$, where $X_{\Omega}$ is the matrix $X$ with zeros filled in at the unobserved locations.
  \STATE Initialize $\hat{g}^0(z)=\frac{\siom}{mn}z$ 
   \FOR{$t=1,\ldots,T$}
   \STATE $\hat{Z}_{i,j}^{t}\leftarrow \hat{Z}_{i,j}^{t-1}-\eta (\hat{g}^{t-1}(\hat{Z}_{i,j}^{t-1})-X_{i,j})\ones_{(i,j)\in\Omega}$
   \STATE $\hat{Z}^t\leftarrow \cP_{r}(\hat{Z}^t)$
   \STATE Solve the optimization problem in ~\eqref{eqn:lpav} to get $\hat{X}$
   \STATE Set $\hat{g}^t(\hat{Z}_{i,j}^t)=\hat{X}_{i,j}$ for all $(i,j)\in \Omega$. 
   \ENDFOR
   \STATE Obtain $\hat{g}^T$ on the entire real line using linear interpolation shown in equation~\eqref{eqn:interpolate}.
  \end{algorithmic}
\end{algorithm}
Let us now explain our initialization procedure. Define $X_{\Omega}\defeq\sum_{j=1}^{\siom} X\circ \Delta_j$, where each $\Delta_j$ is a boolean mask with zeros everywhere and a $1$ at an index corresponding to the index of an observed entry. $A\circ B$ is the Hadamard product, i.e. entry-wise product of matrices $A,B$. We have $\siom$ such boolean masks each corresponding to an observed entry. We initialize $\hat{Z}^0_{\Omega}$ to $\frac{mn}{\siom}X_{\Omega}=\frac{mn}{\siom}\sum_{j=1}^{\siom} X\circ \Delta_j$. Because each observed index is assumed to be sampled uniformly at random with replacement, our initialization is guaranteed to be an unbiased estimate of $X$. 
\section{MSE Analysis of  \mmc}
\label{sec:analysis}
We shall analyze our algorithm, $\limco$, for the case of $T=1$, under the modeling assumption shown in Equations~\eqref{eqn:model1} and ~\eqref{eqn:model2}. Additionally, we will assume that the matrices $\Zstar$ and $\Mstar$ are bounded entry-wise in absolute value by $1$. When $T=1$, the $\limco$ algorithm estimates $\Zhat$, $\ghat$ and $\Mhat$ as follows 
\begin{equation}
\label{eqn:proj_z}
\Zhat=\cP_r\left(\frac{mn\Xom}{\siom}\right).
\end{equation}
$\ghat$ is obtained by solving the LPAV problem from Equation~\eqref{eqn:lpav} with $\Zhat$ shown in Equation~\eqref{eqn:proj_z}. This allows us to define $\Mhat_{i,j}=\ghat(\Zhat_{i,j}), \forall i=[n], j=[m]$.

Define the mean squared error (MSE) of our estimate $\Mhat$ as 
\begin{equation}
\MSE(\Mhat)=\bbE \left[\frac{1}{mn}\sum_{i=1}^n \sum_{j=1}^m (\Mhatij-\Mij)^2\right].
\end{equation}
Denote by $||M||$ the spectral norm of a matrix $M$. We need the following additional technical assumptions:
\begin{itemize}
\item[A1.] $\|\Zstar\|=O(\sqrt{n})$.
\item[A2.] $\sigma_{r+1}(X)=\tilde{O}(\sqrt{n})$ with probability at least $1-\delta$, where $\tilde{O}$ hides terms logarithmic in $1/\delta$.
\end{itemize}
$\Zstar$ has entries bounded in absolute value by $1$. This means that in the worst case, $||\Zstar||=\sqrt{mn}$. Assumption A1 requires that the spectral norm of $\Zstar$ is not very large. Assumption A2 is a weak assumption on the decay of the spectrum of $\Mstar$. By assumption $X=\Mstar+N$. Applying  Weyl's inequality we get $\sigma_{r+1}(X)\leq \sigma_{r+1}(\Mstar)+\sigma_1(N)$. Since $N$ is a zero-mean noise matrix with independent bounded entries, $N$ is a matrix with sub-Gaussian entries. This means that $\sigma_1(N)=\tilde{O}(\sqrt{n})$ with high probability. Hence, assumption A2 can be interpreted as imposing the condition $\sigma_{r+1}(\Mstar)=O(\sqrt{n})$. This means that while $\Mstar$ could be full rank, the $(r+1)^{\text{th}}$ singular value of $\Mstar$ cannot be too large.
\begin{thm}
\label{thm:main}
Let $\mu_1\defeq \bbE ||N||, \mu_2\defeq \bbE ||N||^2$. Let $\alpha=||\Mstar-\Zstar||$. Then, under assumptions A1 and A2, the MSE of the estimator output by $\mmc$ with $T=1$ is given by
\begin{equation}
\label{eqn:bound}
\begin{split}
\MSE(\Mhat)&=O\Biggl(\sqrt{\frac{r}{m}}+\frac{\sqrt{mn\log(n)}}{\siom}+\frac{mn}{\siom^{3/2}}+\sqrt{\frac{r}{m\sqrt{n}}\left(\mu_1+\frac{\mu_2}{\sqrt{n}}\right)}+\\
&\qquad \sqrt{\frac{r\alpha}{m\sqrt{n}}\left(1+\frac{\alpha}{\sqrt{n}}\right)}+
 \sqrt{\frac{rmn\log^2(n)}{\siom^2}}\Biggr).
\end{split}
\end{equation}
\end{thm} 
where $O(\cdot)$ notation hides universal constants, and the Lipschitz constant $L$ of $\gstar$. 
 We would like to mention that the result derived for $\mmc$-1 can be made to hold true for $T>1$, by an additional large deviation argument. The proof of our theorem is available in the appendix.

\textbf{Interpretation of our results:} Our upper bounds on the MSE of $\limco$ depends on the quantity $\alpha=||\Mstar-\Zstar||$, and $\mu_1,\mu_2$. Since matrix $N$ has independent zero-mean entries which are bounded in absolute value by $1$, $N$ is a sub-Gaussian matrix with independent entries. For such matrices $\mu_1=O(\sqrt{n}),\mu_2=O(n)$ (see Theorem 5.39 in~\cite{vershynin2010introduction}). With these settings we can simplify the expression in Equation~\eqref{eqn:bound} to
\begin{equation*}
\begin{split}
\label{eqn:samp_comp}
\MSE(\Mhat)&=\tilde{O}\Biggl(\sqrt{\frac{r}{m}}+\frac{\sqrt{mn\log(n)}}{\siom}+\frac{mn}{\siom^{3/2}}+ \sqrt{\frac{r\alpha}{m\sqrt{n}}\left(1+\frac{\alpha}{\sqrt{n}}\right)}+
 \sqrt{\frac{rmn\log^2(n)}{\siom^2}}\Biggr).
 \end{split}
\end{equation*}
A remarkable fact about our sample complexity results is that the sample complexity is independent of the rank of matrix $\Mstar$, which could be large. Instead it depends on the rank of matrix $\Zstar$ which we assume to be small. The dependence on $\Mstar$ is via the term $\alpha=||\Mstar-\Zstar||$.  From equation~\eqref{eqn:samp_comp} it is evident that the best error guarantees are obtained when $\alpha=O(\sqrt{n})$. For such values of $\alpha$ equation~\eqref{eqn:samp_comp} reduces to,
\begin{equation*}
\begin{split}
\MSE(\Mhat)=\tilde{O}\Biggl(\sqrt{\frac{r}{m}}+\frac{\sqrt{mn\log(n)}}{\siom}+\frac{mn}{\siom^{3/2}}+\frac{\sqrt{mn}}{\siom}+ \sqrt{\frac{rmn\log^2(n)}{\siom^2}}\Biggr).
 \end{split}
\end{equation*}
 This result can be converted into a sample complexity bound as follows. If we are given $\siom=\tilde{O}(\left(\frac{mn}{\epsilon}\right)^{2/3})$, then $\MSE(\Mhat)\leq \sqrt{\frac{r}{m}}+\epsilon.$

\section{Experimental results}
\label{sec:expts}
We compare the performance of $\limco-1$, $\limco-c$, $\limco$- LS, and nuclear norm based low-rank matrix completion (LRMC)~\cite{candes2009exact} on various synthetic and real world datasets. The objective metric that we use to compare different algorithms is the root mean squared error (RMSE) of the algorithms on unobserved, test indices of the incomplete matrix. 

\subsection{Efficient implementation of $\limco$}
The $\limco$ algorithm consists of three main steps. A gradient descent step, a projection step onto the set of rank $r$ matrices, and a QP optimization step. On top of these one has to find out the correct values for the rank $r$, and the step size $\eta$ used in step 4 of $\limco$. A straight forward implementation of the above three steps can be coupled with a grid search over $r$ and $\eta$. However, such an implementation would be inefficient and not scalable beyond small matrices. Since $r$ is assumed to be small one needs to find just the top few singular vectors of the intermediate matrices. We use PROPACK~\footnote{Look at~\url{http://svt.stanford.edu/code.html} for download instructions.} package for an efficient SVD implementation. Rather than search for $r$ on a grid, we use an increasing rank procedure to estimate $r$.  The increasing rank procedure starts at a small value for $r$, say $r_{\text{min}}$ and increases the current estimate of $r$, by $r_{\text{inc}}$, whenever there is not much progress in the iterates. Precisely, the current estimate of $r$ is increased by $r_{\text{inc}}$ if 
\begin{equation}
1-\frac{||P_{\Omega}(\hat{X}^t-X)||_F}{||P_{\Omega}(\hat{X}^{t-1}-X)||_F}\leq \epsilon,
\end{equation}
for some small $\epsilon>0$. As we expect $r$ to be small, our estimate for $r$ is not allowed to grow beyond a certain $r_{\text{max}}$. The increasing rank procedure was inspired by the work of~\cite{wen2012solving}, where they demonstrated that such procedures are suitable for low-rank approximation problems. 

\textbf{An ADMM implementation for LPAV problem.} The LPAV problem shown in Equation~\eqref{eqn:lpav} is solved in each iteration of $\limco$. Hence, it is crucial that we have an efficient implementation of LPAV so that the overall algorithm is efficient. For ease of notation, we shall frame the LPAV problem as follows: Suppose we have a 1-dimensional regression problem with covariates $z_1\leq z_2\leq \ldots \leq z_f$ and the corresponding targets $x_1, x_2,\ldots x_f$ in $\bbR$. The LPAV routine solves the problem
\begin{equation}
\begin{aligned}
\label{eqn:lpav_admm}
    \underset{y\in \bbR^f}{\text{min}}&~ \sum_{i=1}^f (y_i-x_i)^2\\
  \text{subject to:} &~0\leq y_{i+1}-y_i\leq L(z_{i+1}-z_{i}), \text{for all~} i=1,\ldots,f-1.
   \end{aligned}
   \end{equation}    
   One could transform the above optimization problem into a box-constrained convex QP, by rewriting the objective and the constraints using the transformation $\delta_i=y_{i+1}-y_i$. Box constrained QPs are well studied optimization problems and a plethora of efficient solvers exist. However, the above transformation destroys the simple least-squares structure of the objective, and in our experience the resulting box-constrained QP turned out to be harder to solve than the original optimization problem~\eqref{eqn:lpav_admm}. Instead we use the Alternating Direction Method of Multipliers (ADMM) algorithm for our problem~\cite{boyd2011distributed} which allows us to exploit the rich structure present in problem~\eqref{eqn:lpav_admm}. We shall now explain ADMM updates as applied to problem~\eqref{eqn:lpav_admm}.  Let $b$ be a vector in $\bbR^f$, with the $j^{\text{th}}$ component being $b_j=L(z_{j+1}-z_{j})$. Let $M_1, M_2$ be two matrices in $\bbR^{f-1\times f}$, constructed as follows. The $i^{th}$ row of $M_1$ has $1$ in the $i^{th}$ column and $-1$ in the $i+1^{th}$ column. The $i^{th}$ row of $M_2$ has $-1$ in the $i^{th}$ column and $1$ in the $i+1^{th}$ column. Rest of the entries in $M_1, M_2$ are zeros. Let 
   $M=
   \bigl(\begin{smallmatrix}
	M_1\\ M_2
\end{smallmatrix}\bigr)$ be a matrix in $\bbR^{2p-2\times p}$. Finally let $\bar{b}=\bigl[\begin{smallmatrix}0 &\\ b &\end{smallmatrix}\bigr]$ be a $2f-2$ dimensional vector whose first $f-1$ entries are $0$, and the last $f-1$ entries are the first $f-1$ entries of the vector $b$. With these definitions, problem~\eqref{eqn:lpav_admm} can be reformulated as
\begin{equation}
\label{eqn:reform1}
\begin{aligned}
    \underset{y\in \bbR^f}{\text{min}}&~ \sum_{i=1}^f (y_i-x_i)^2\\
  \text{subject to:}&~ My\leq \bar{b}.
   \end{aligned}
   \end{equation}   
   In order to derive efficient ADMM updates, we shall convert the inequality constraints to equality by introducing slack variables. This gets us the following reformulation

   \begin{equation}
\label{eqn:lpav_admm2}
\begin{aligned}
    \underset{\bar{y}\in \bbR^{3f-2}}{\text{min}}&~ \sum_{i=1}^f (\bar{y}_i-x_i)^2\\
  \text{subject to:}&~ \bar{M}\bar{y}= \bar{b}\\
  &[0_{f\times f},I_{f\times f}]\bar{y}\geq 0
   \end{aligned}
   \end{equation}
   where $\bar{M}=[M_{2p-2\times p}, I_{2f-2\times 2f-2}]$. Note that this problem has $3f-2$ variables in contrast to the $f-1$ variables in Equation~\eqref{eqn:reform1}. The extra $2f-1$ variables correpsond to the slack variables introduced when converting the inequality constraints in~\eqref{eqn:reform1} to equality constraints in~\eqref{eqn:lpav_admm2}. 
The ADMM algorithm provides an iterative procedure for solving convex optimization problems of the form
\begin{equation}
\begin{aligned}
\underset{\bar{y},z}{\text{min}}~ & f(\bar{y})+g(z)\\
\text{subject to:}~& A\bar{y}+Bz=0.
\end{aligned}
\end{equation}
We can cast the optimization problem in Equation~\eqref{eqn:lpav_admm2} in the above form by defining $A=-B=I$,
\begin{equation}
f(\bar{y})=
\begin{cases}
\sum_{i=1}^f (\bar{y}_i -x_i)^2~\text{if}~ \bar{M}\bar{y}=b\\
\infty~\text{otherwise}
\end{cases}
\end{equation}
and 
\begin{equation*}
g(z)=
\begin{cases}
0~\text{if} ~[0_{f\times f}, I_{f\times f}]z\geq 0\\
\infty~\text{otherwise}.
\end{cases}
\end{equation*}
The ADMM algorithm involves iteratively updating $\bar{y},z$ and some additional variables until convergence. In our practical implementations we follow the advice as mentioned in~\cite{boyd2011distributed} and set $\epsilon^{\text{abs}}$ and $\epsilon^{\text{rel}}$ to $10^-2$. Any value less than $10^-2$ did not yield much change in results.

\textbf{Updating $\bar{y}$:} In order to update $\bar{y}$ iteratively we need to solve the following optimization problem. Given two vectors $z^k, u^k$, and $\gamma>0$, we solve
\begin{equation}
\begin{aligned}
\bar{y}^{k+1}=\arg\min_{\bar{y}}& \sum_{i=1}^f (\bar{y}_i-x_i)^2 +\frac{\gamma}{2} ||\bar{y}+u^k-z^k||^2\\
\text{subject to:}&~ \bar{M}\bar{y}=\bar{b}.
\end{aligned}
\end{equation}
This is an equality constrained QP, which can be solved in  closed form by using the KKT conditions. The solution is obtained by solving the following set of sparse linear equations.
\begin{equation}
\label{eqn:system}
\begin{bmatrix}
P_{3f-2\times 3f-2}+\gamma I_{3f-2\times 3f-2}& \bar{M}^\top\\
\bar{M}& 0_{2f-2\times 2f-2} 
\end{bmatrix}
\begin{bmatrix}
\bar{y}^{k+1}\\
\nu
\end{bmatrix}=
\begin{bmatrix}
-(q+\gamma (u^k-z^k))\\
\bar{b},
\end{bmatrix}
\end{equation}
where $P,q$ are defined as follows
\begin{align}
P=2\begin{bmatrix}
I_{f\times f}& 0_{f\times 2f-2}\\
0_{2f-2\times f}& 0_{2f-2\times 2f-2} \end{bmatrix}
, q=-2\begin{bmatrix}
x\\
0_{2p-2}
\end{bmatrix}
\end{align}
The above system of linear equations is very large with $5p-4$ equations in $5p-4$ unknowns but very sparse with $7p-6$ non-zeros in the matrix on the LHS of Equation~\eqref{eqn:system}. Such large, sparse system of equations can be solved very efficiently, and in our numerical experiments we simply use the backslash operator in MATLAB to solve them.

\textbf{Updates for $z$:} The ADMM updates applied to $z$ yield the following optimization problem
\begin{equation}
\begin{aligned}
z^{k+1} =\arg\min_{z\in \bbR^{3p-2}, } ||z-(y^{k+1}+u^{k})||^2.\\
\text{subject to:}~ z_k \geq 0~\text{for}~ k\geq p+1
\end{aligned}
\end{equation}
The solution to the above optimization problem is trivial and can be written in closed form as  follows
\begin{equation}
z_j^{k+1}=
\begin{cases}
y^{k+1}_j+u^k_j ~\text{if}~ j=1,\ldots p\\
\max\{y^{k+1}_j+u^k_j, 0\}~\text{if}~ j\geq p+1
\end{cases}
\end{equation}

\textbf{Updates for $u$:} The update for the intermediate variable $u$ is given by the equation
\begin{equation}
u^{k+1} = u^k+\bar{y}^{k+1}-z^{k+1}.
\end{equation}
Hence, the ADMM algorithm for the LPAV routine requires iteratively updating $\bar{y},z,u$ until desired tolerance levels.

\subsection{Synthetic experiments}
For our synthetic experiments we generated a random $30\times 20$ matrix $\Zstar$ of rank $5$ by taking the product of two random Gaussian matrices of size $n\times r$, and $r\times m$, with $n=30,m=20,r=5$. The matrix $\Mstar$ was generated using the function, $\gstar(\Mstarij)=1/(1+\exp(-c\Zstarij))$, where $c>0$. By increasing $c$, we increase the Lipschitz constant of the function $\gstar$, making the matrix completion task harder.  For large enough $c$, $M_{i,j}\approx \sgn(Z_{i,j})$. We consider the noiseless version of the problem where $X=\Mstar$. Each entry in the matrix $X$ was sampled with probability $p$, and the sampled entries are observed. This makes $\bbE \siom=mnp$. For our implementations we assume that $r$ is unknown, and estimate it either (i) via the use of a dedicated validation set in the case of $\mmc-1$ or (ii) adaptively, where we progressively increase the estimate of our rank until a sufficient decrease in error over the training set is achieved~\cite{wen2012solving}. For an implementation of the LRMC algorithm we used a standard off-the-shelf implementation from TFOCS~\cite{becker2011tfocs}. In order to speed up the run time of $\mmc$, we also keep track of the training set error, and terminate iterations if the relative residual on the training set goes below a certain threshold~\footnote{For our experiments this threshold is set to $0.001$.}. In the appendix we provide a plot that demonstrates that, for $\mmc-c$, the RMSE on the training dataset has a decreasing trend and reaches the required threshold in at most 50 iterations. Hence, we set $T=50$. Figure~\eqref{fig:synth} show the RMSE of each method for different values of $p,c$.
\begin{figure}[t!]
 \begin{subfigure}[b]{0.15\columnwidth}
    \includegraphics[height=1.0in]{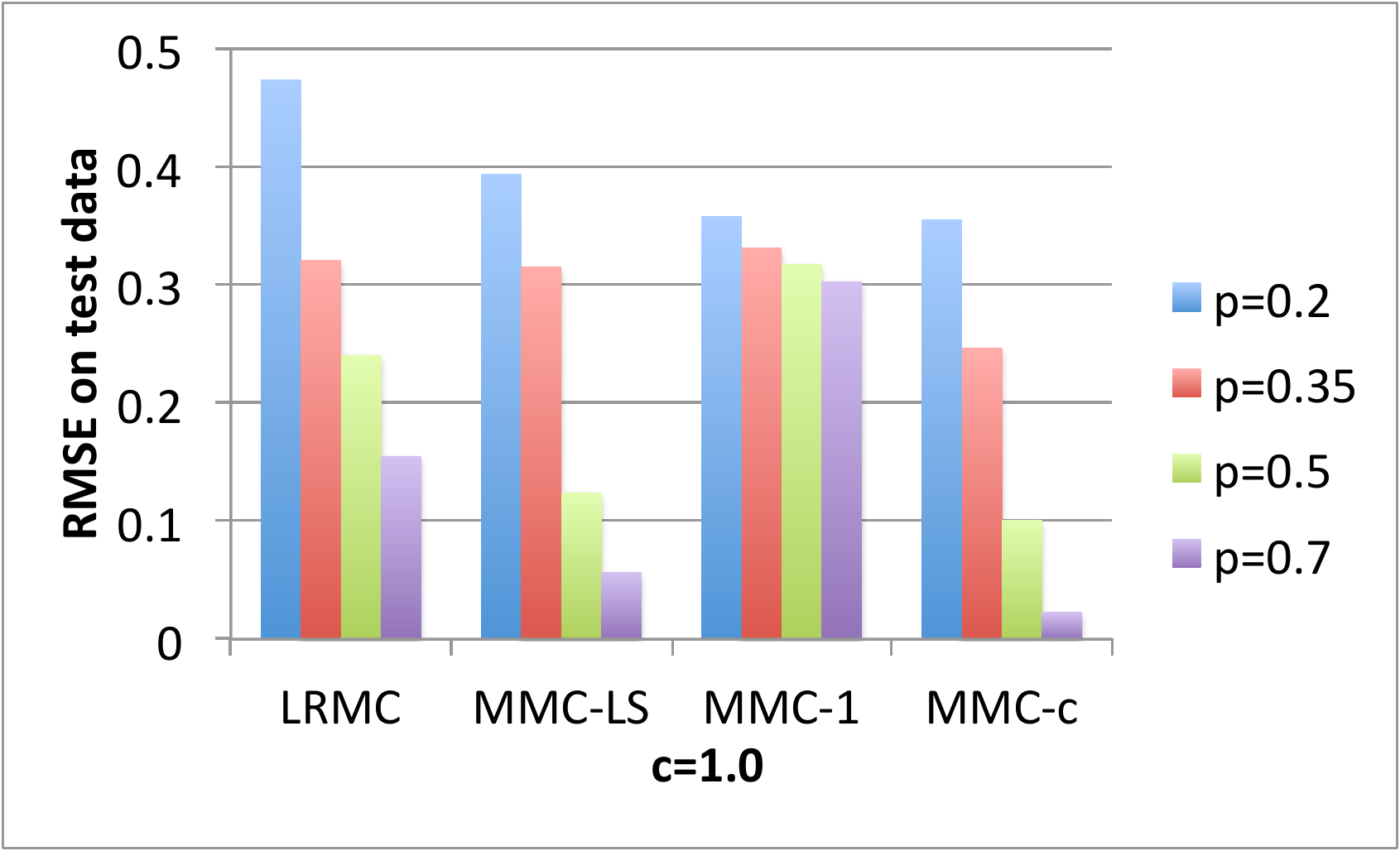}
 \end{subfigure}
 \hspace{0.99in}
 \begin{subfigure}[b]{0.15\columnwidth}
    \includegraphics[height=1.0in]{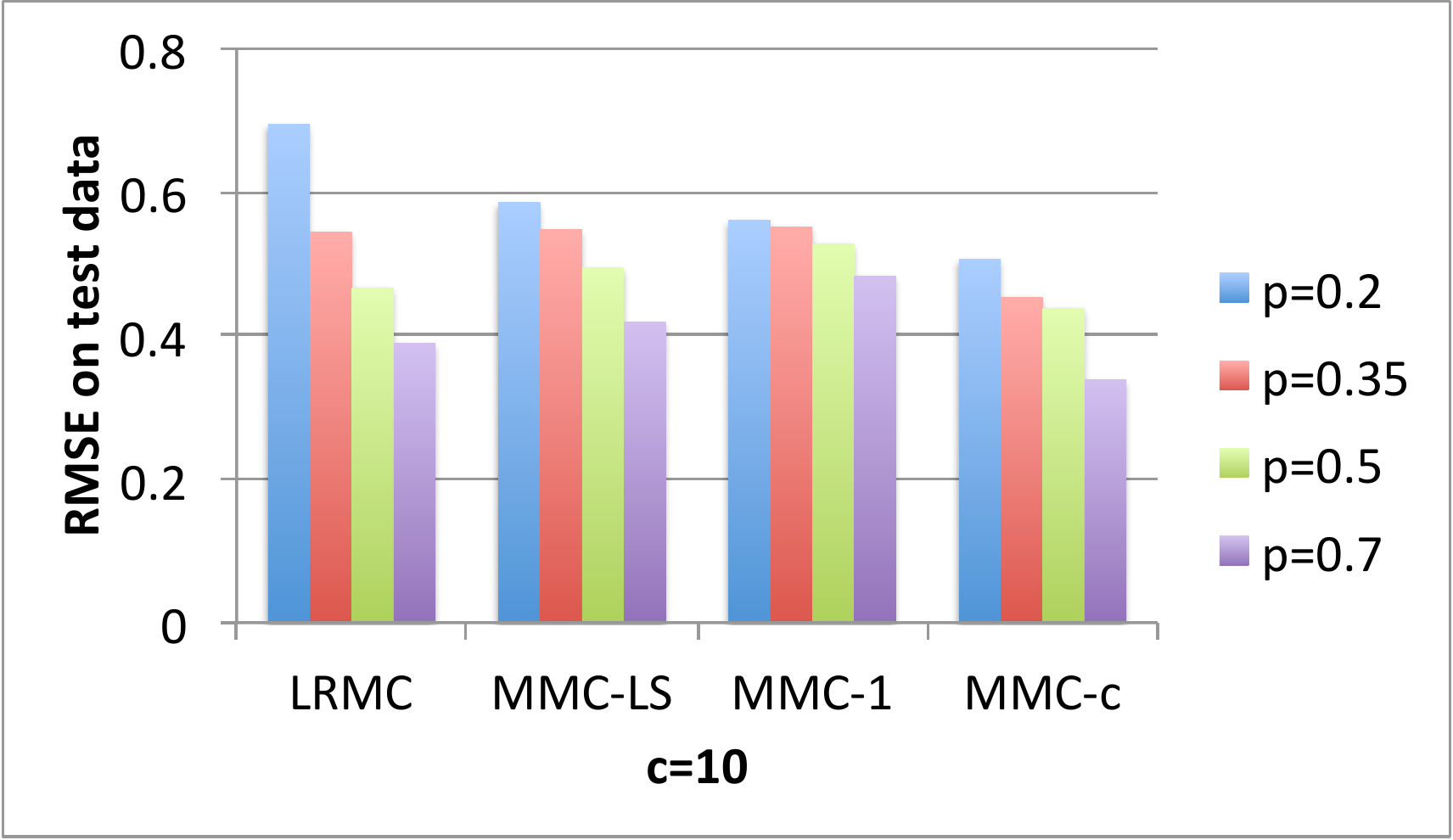}
 \end{subfigure}
 \hspace{0.99in}
 \begin{subfigure}[b]{0.15\columnwidth}
    \includegraphics[height=1.0in]{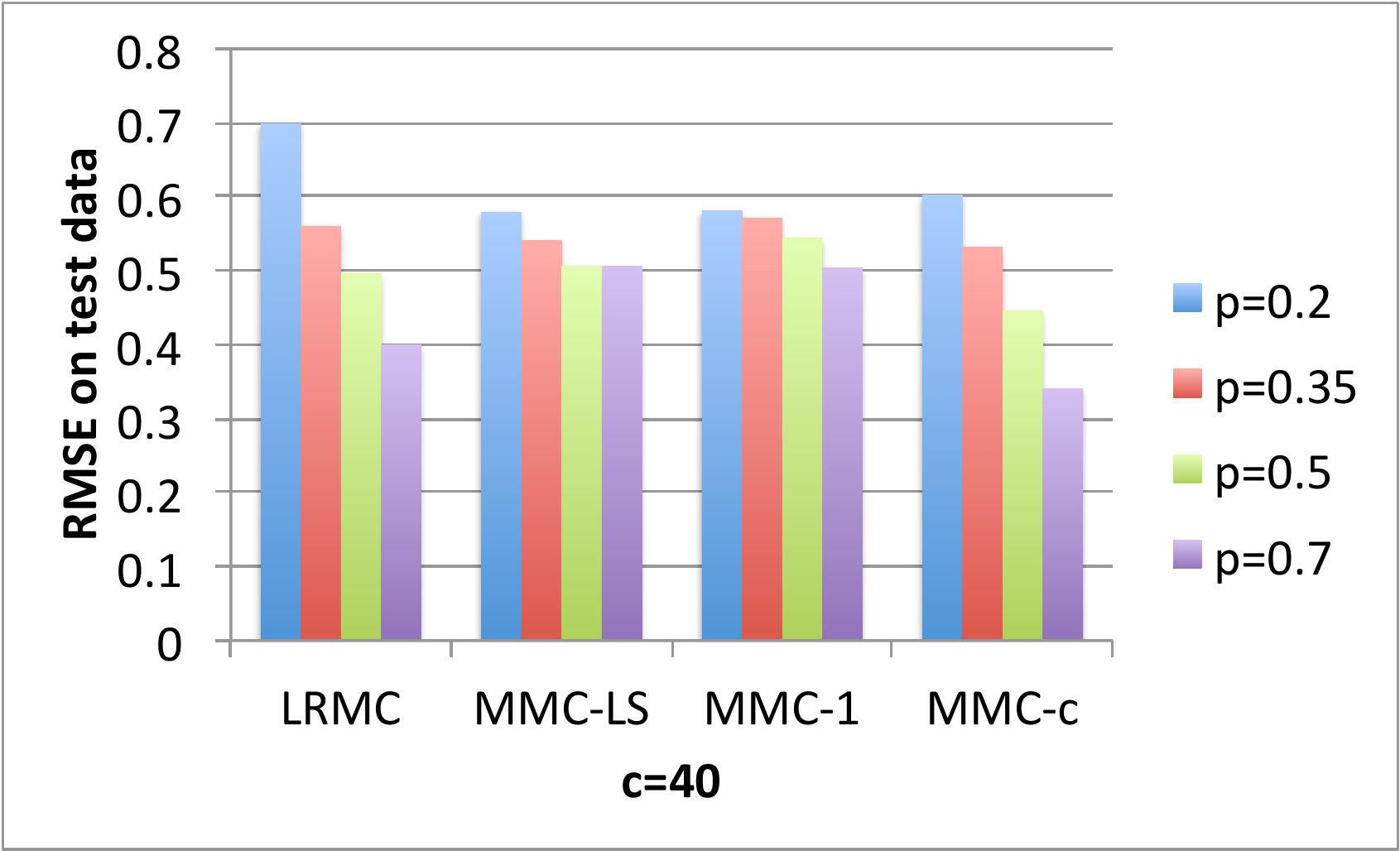}
 \end{subfigure}
\caption{RMSE of different methods at different values of $c$.\label{fig:synth}}
\vspace{-15pt}
\end{figure}
As one can see from figure~\eqref{fig:synth}, the RMSE of all the methods improves for any given $c$ as $p$ increases. This is expected since as $p$ increases $\bbE\siom=pmn$ also increases. As $c$ increases, $\gstar$ becomes  steeper increasing the effective rank of $X$. This makes matrix completion task hard. For small $p$, such as $p=0.2$, $\mmc-1$ is competitive with $\mmc-c$ and $\mmc-$LS and is often the best. In fact for small $p$, irrespective of the value of $c$, LRMC is far inferior to other methods. For larger $p$, $\mmc-c$ works the best achieving smaller RMSE over other methods.
\vspace{-5pt}
\subsection{Experiments on real datasets}
We performed experimental comparisons on four real world datasets: paper recommendation, Jester-3, ML-100k, Cameraman. The source of our datasets is listed in the appendix. All of the above datasets, except the Cameraman dataset, are ratings datasets, where users have rated a few of the several different items. For the Jester-3 dataset we used $5$ randomly chosen ratings for each user for training, $5$ randomly chosen rating for validation and the remaining for testing. ML-100k comes with its own training and testing dataset. We used 20\% of the training data for validation. For the Cameraman and the paper recommendation datasets 20\% of the data was used for training, 20\% for validation and the rest for testing.  The baseline algorithm chosen for low rank matrix completion is LMaFit-A~\cite{wen2012solving}~\footnote{\url{http://lmafit.blogs.rice.edu/}. The parameter $k$ in the LMaFit algorithm was set to effective rank, and we used est\_rank=1 for LMaFit-A.}. 

For each of the datasets we report the RMSE of $\mmc-1$, $\mmc-c$, and LMaFit-A on the test sets. We excluded $\mmc$-LS from these experiments because in all of our datasets the number of observed entries is a very small fraction of the total number of entries, and from our results on synthetic datasets we know that $\mmc-$ LS is not the best performing algorithm in such cases. Table~\ref{tab:real} shows the RMSE over the test set of the different matrix completion methods. As we see the RMSE of $\mmc-c$ is the smallest of all the methods, surpassing LMaFit-A by a large margin.

\begin{table}[H]
\centering
\caption{\label{tab:real}RMSE of different methods on real datasets.} 
\begin{tabular}{|l|l|l|l|l|l|l|}
\hline
Dataset&Dimensions&$\siom$&$r_{0.01}(X)$&LMaFit-A&$\limco-1$& $\limco-c$\\
\hline
PaperReco&$3426\times 50$&34294&47&0.4026&0.4247&\textbf{0.2965}\\
\hline
Jester-3&$24938\times 100$&124690&66&6.8728&5.327&\textbf{5.2348}\\
\hline
ML-100k&$1682\times 943$&64000&391&3.3101&1.388&\textbf{1.1533}\\
\hline
Cameraman&$1536\times 512$&157016&393&0.0754&0.1656&\textbf{0.06885}\\

\hline
\end{tabular}
\end{table}
\section{Conclusions and future work}
\label{sec:conc}
We have investigated a new framework for high rank matrix completion problems called monotonic matrix completion. We proposed and studied an algorithm called $\limco$ based on minimizing a calibrated loss function. In the future we would like to investigate if one could relax the technical assumptions involved in establishing our theoretical results.

\bibliographystyle{unsrt}
\small{\bibliography{mmc}}

\appendix

\section{Error Analysis of Monotonic Matrix Completion}
We shall analyze our algorithm, $\mmc-$c, for the case of $T=1$. Since for $T=1$, $\mmc-$c and $\mmc-$LS are the same, we shall used the word $\mmc$ to refer to both the algorithms when $T=1$. For $T=1$, we have
\begin{align}
\Zhat&=P_r\left(\frac{mn\Xom}{\siom}\right)\\
\ghat&=\LPAV(\Zhatom,\Xom)\\
\Mhat_{i,j}&=\ghat(\Zhat_{i,j}), \forall i=[m], j=[n],
\end{align}
Finally, define the mean squared error (MSE) of our estimate $\Mhat$ can be defined as 
\begin{equation}
\MSE(\Mhat)=\bbE \left[\frac{1}{mn}\sum_{i=1}^n \sum_{j=1}^m (\Mhatij-\Mij)^2\right].
\end{equation}
We are interested in analyzing the $\MSE$ of $\Mhat$ output by $\mmc$ for $T=1$. We shall make the following assumptions
\section{MMC model and technical assumptions}
\begin{itemize}
\item[A1] $\|\Zstar\|=O(\sqrt{n})$, i.e. the spectral norm of $\Zstar$ is of the order of $\sqrt{n}$.
\item[A2.] $\sigma_{r+1}(X)=O(\sqrt{n})$ with probability at least $1-\delta$.
\end{itemize}
The MMC model makes the following assumptions. These assumptions are the same as in the main paper. We enumerate it here for the sake of convenience.
\begin{itemize}
\item[M1.] $X=\Mstar+N$.
\item[M2.] $\bbE N=0$.
\item[M3.] $\Mstarij=\gstar(\Zstarij)~\forall i=[n],j=[m]$.
\item[M4.] Assume that $n\geq m$, and $\rank(\Zstar)=r\ll m$.
\item[M5.] Boundedness assumption: $|\Zstarij|\leq 1,|\Xij|\leq 1$ for all $i\in [n],j\in [m]$.
\item[M6.] $\gstar:\bbR\rightarrow \bbR$ is monotonic and $L$-Lipschitz.
\item[M7.] The set $\Omega$ is generated by sampling uniformly at random with replacement from the index set $[n]\times[m]$.
\end{itemize}
\subsection{Notation}
All of our matrices, unless explicitly stated, will be $n\times m$ with $n\geq m$. $||A||$ is the spectral norm of matrix $A$, and $||A||_{\star}$ is the nuclear norm of matrix $A$. 
\subsection{Towards proof of Theorem~\eqref{thm:main}}
We begin with the following technical lemma that will be used in the proof.
\begin{lemma} 
\label{lem:large_dev}
Let $\cG=\{g|g:[-W,W]\rightarrow [-1,1] ~\text{is monotonic and 1-Lipschitz}\}$.With probability at least $1-\delta$ over the sample $z_1,\ldots,z_n$, the following statement is true for all $g\in \cG$
\begin{equation}
\left|\frac{1}{n}\sum (g(z_i)-y_i)^2-\bbE(g(z)-y)^2\right|=\tilde{O}\left(\sqrt{\frac{W}{n}}\right)
\end{equation}
where $\tilde{O}$ hides logarithmic dependence on $n,W,1/\delta$.
\end{lemma}
\begin{proof}
Let $\hat{\cR}_n(\cG)$ be the empirical Rademacher complexity of function class $\cG$, and let $\cN_{\infty}(\epsilon,\cG)$ be the $L_{\infty}$ covering number of the function clas $\cG$. From~\citep[Lemma 6]{kakade2011efficient} we know that
\begin{equation}
\cN_{\infty}(\epsilon,\cG)\leq \frac{1}{\epsilon}2^{\frac{2W}{\epsilon}}.
\end{equation}
The above covering number allows us to bound the empirical Rademacher complexity of the function class $\cG$ via Dudley's entropy bound. Using~\citep[Lemma A.3]{srebro2010smoothness}, and the fact that $\cN_{\infty}(\epsilon,\cG)\geq \cN_2(\epsilon,\cG,z_1,\ldots,z_n)$ we get
\begin{align}
\hat{\cR}_n(\cG)&\leq \inf_{\alpha\geq 0} 4\alpha+10\int_{\alpha}^1 \sqrt{\frac{\log \cN_{\infty}(\epsilon,\cG)}{n}}~\mathrm{d}\epsilon\\
&\leq 4\alpha+10\int_{\alpha}^1 \sqrt{\frac{\frac{2W}{\epsilon}\log(\frac{1}{\epsilon})}{n}}~\mathrm{d}\epsilon\\
&\leq 4\alpha+10\sqrt{\frac{2W}{n}}\int_{\alpha}^1 \frac{1}{\epsilon}~\mathrm{d}\epsilon\\
&\leq 10\sqrt{\frac{2W}{n}}\log\left(\frac{4e}{10}\sqrt{\frac{n}{2W}}\right).
\end{align}
Using a uniform convergence bound in terms of the Rademacher complexity of the function class~\citep[Theorem 8]{bartlett2003rademacher} we get the desired result.
\end{proof}
\begin{lemma}
\label{lem:eps_delta}
Let $\epsilon_2=\bbE [\frac{1}{mn}\sum_{i,j}(\Zhatij-\Zstarij)^2]$. Then, under assumptions A1-A8, we have 
\begin{equation*} 
\MSE(\Mhat)\leq O\left(\frac{\sqrt{mn\log(n)}}{\siom}+\sqrt{\frac{n}{\siom}}+\frac{mn}{\siom^{3/2}}+\frac{\sqrt{mn}}{\siom}+\epsilon_2+\sqrt{\epsilon_2}\right)
\end{equation*}
\end{lemma}
\begin{proof}
\begin{align}
\frac{1}{mn}\bbE \left[\sum_{i,j}(\Mhatij-\Mstarij)^2\right]&=\frac{1}{mn}\bbE \left[\sum_{i,j}(\ghat(\Zhatij)-\gstar(\Zstarij))^2\right]\\
&=\frac{1}{mn}\bbE \left[\sum_{i,j}\left(\ghat(\Zhatij)-\gstar(\Zhatij)+\gstar(\Zhatij)-\gstar(\Zstarij)\right)^2\right]\\
&\leq2\underbrace{\bbE \left[\frac{1}{mn}\sum_{i,j}\left(\ghat(\Zhatij)-\gstar(\Zhatij)\right)^2\right]}_{T_1}+2\underbrace{\bbE\left[\frac{1}{mn}\left(\gstar(\Zhatij)-\gstar(\Zstarij)\right)^2\right]}_{T_2}\\
&=2T_1+2T_2.
\end{align}
We shall bound $T_2$ in terms of $\epsilon_2$.

\textbf{Bounding $T_2$:} 
\begin{align}
T_2&=\frac{1}{mn}\bbE \sum_{i,j} (\gstar(\Zhatij)-\gstar(\Zstarij))^2\\
&\leqa\frac{1}{mn}\bbE \sum_{i,j}(\Zhatij-\Zstarij)^2\\
&\defeq \epsilon_2
\end{align}
where inequality (a) follows from the fact that $\gstar$ is 1-Lipschitz.
Next we shall bound $T_1$ in terms of $\epsilon_2$ and other terms.

\textbf{Bounding $T_1$:} 
\begin{equation}
\begin{split}
\bbE \left[\frac{1}{mn}\sum_{i,j}\left(\ghat(\Zhatij)-\gstar(\Zhatij)\right)^2\right]&= \underbrace{\bbE\left[\frac{1}{\siom}\sum_{\Omega}\left(\ghat(\Zhatij)-\gstar(\Zhatij)\right)^2\right]}_{T_{1,1}}+\\
&\quad \underbrace{\bbE \left[\frac{1}{mn}\sum_{i,j}\left(\ghat(\Zhatij)-\gstar(\Zhatij)\right)^2\right]-\bbE\left[\frac{1}{\siom}\sum_{\Omega}\left(\ghat(\Zhatij)-\gstar(\Zhatij)\right)^2\right]}_{\Delta_1}
\end{split}
\end{equation}
Next we shall bound $T_{1,1}$  as follows. Let $\cG=\{g:\bbR\rightarrow \bbR| g \text{~is a monotonic  and 1-Lipschitz function~}\}$. Since $\ghat,\gstar$ by definition belong to $\cG$, and since $\ghat$ solves the optimization problem
\begin{equation}
\ghat=\arg\min \sum_{\Omega} (g(\Zhatij)-\Xij)^2,
\end{equation}
hence via the generalized Pythagorean inequality~\cite{cesa2006prediction} we have 
\begin{equation}
\label{eqn:gpt}
\sum_{\Omega} (\ghat(\Zhatij)-\Xij)^2+\sum_{\Omega}(\ghat(\Zhatij)-\gstar(\Zhatij))^2\leq \sum_{\Omega} (\Xij-\gstar(\Zhatij))^2.
\end{equation}
Using Equation~\eqref{eqn:gpt} we can bound $T_{1,1}$ as follows
\begin{equation}
\begin{split}
T_{1,1}&=\bbE\left[\frac{1}{\siom}\sum_{\Omega}\left(\ghat(\Zhatij)-\gstar(\Zhatij)\right)^2\right]\\
&\leq \bbE\left[\frac{1}{\siom}\sum_{\Omega}(\Xij-\gstar(\Zhatij))^2-\frac{1}{\siom}\sum_{\Omega}\left(\Xij-\ghat(\Zhatij)\right)^2\right]\\
&=\underbrace{\bbE\left[\frac{1}{\siom}\sum_{\Omega}(\Xij-\gstar(\Zhatij))^2\right]- \bbE\left[\frac{1}{\siom}\sum_{\Omega}(\Xij-\gstar(\Zstarij))^2\right]}_{I_1}+\\
&\quad\underbrace{\bbE\left[\frac{1}{mn}\sum_{i,j}(\Xij-\gstar(\Zstarij))^2\right]-\bbE\left[\frac{1}{mn}\sum_{i,j}(\Xij-\ghat(\Zhatij))^2\right]}_{I_2}+\\
&\quad \underbrace{\bbE\left[\frac{1}{\siom}\sum_{\Omega}(\Xij-\gstar(\Zstarij))^2\right]- \bbE\left[\frac{1}{mn}\sum_{i,j}(\Xij-\gstar(\Zstarij))^2\right]}_{I_3}+\\
&\quad \underbrace{\bbE\left[\frac{1}{mn}\sum_{i,j}(\Xij-\ghat(\Zhatij))^2\right]- \bbE\left[\frac{1}{\siom}\sum_{\Omega}(\Xij-\ghat(\Zhatij))^2\right]}_{I_4} 
\end{split}
\end{equation}
We shall look at each of the terms $T_2,T_3,T_4$ and bound them separately. 
From assumption A1 we know that $\gstar(\Zstarij)$ is the best estimator of $\Xij$ in mean squared. Hence, $I_2\leq 0$. We next bound $I_1,I_3,I_4$. 
$|\Zstar|_{\infty}\leq 1$, and $|\Xstar|_{\infty}\leq 1$, hence $|\Xij-\gstar(\Zstarij)|\leq 2$. If we call $\Delta_3$ the random variable whose expectation is $I_3$, then $\Delta_3\leq 4$ surely. Moreover we can apply lemma~\eqref{lem:large_dev} to guarantee that $\Delta_3\leq O\left(\sqrt{\frac{\log(\siom/\delta)}{\siom}}\right)$ with probability at least $1-\delta$. Choose $\delta=\frac{1}{\sqrt{\siom}}$. We then have 
\begin{equation}
I_3=\bbE\Delta_3\leq 4\delta+(1-\delta)O\left(\sqrt{\frac{\log(\siom/\delta)}{\siom}}\right)=O\left(\sqrt{\frac{\log(\siom)}{\siom}}\right).
\end{equation}
Next, we bound $I_4$. This needs a slightly careful treatment, since $\Zhatij$ is random. Let $A=\frac{1}{p\siom}X\circ \Delta$. Let $A=\sum \sigma_i u_iv_i^\top$ be the SVD of $A$ with $\sigma_1\geq \sigma_2\geq \cdots \sigma_m$. By definition $\Zhat=P_r (A)$. Hence, $A-Z=\sum_{i\geq r+1}\sigma_i u_i v_i^\top$. This means that
\begin{align}
|A-\Zhat|_{\infty}&\leq ||A-\Zhat||\nonumber\\
&=||\sum_{i\geq r+1} \sigma_i u_iv_i^\top||\nonumber\\
&=\sigma_{r+1}\nonumber\\
&\leq \sigma_{1}(A-X)+\sigma_{r+1}(X)\label{eqn:this15}
\end{align}
We shall now use the above bound on $\Zhat-A$ to obtain upper bound on $|\Zhat|_{\infty}$ as follows
\begin{align}
|\Zhat|_{\infty}&\leqa |\Zhat-A|_{\infty}+|A-X|_{\infty}+|X|_{\infty}\\
&\leqb ||A-X||+|X|_{\infty}+||A-X||+\sigma_{r+1}(X)\\
&= 2||A-X||+|X|_\infty+\sigma_{r+1}(X)\\
&= 2||A-X||+1+\sigma_{r+1}(X)\\
&\leqc 2||A-X||+1+\sigma_{r+1}(X)\label{eqn:this16}
\end{align}
To obtain inequality (a) we used the triangle inequality, and to obtain inequality (b) we used Equation~\eqref{eqn:this15}. Now, consider the event
\begin{equation}
\cE_1=\left\{||A-X||\leq \frac{2mn\log\left(\frac{m+n}{\delta}\right)}{3\siom}+\sqrt{\frac{2\log(\frac{m+n}{\delta})mn}{\siom}}\right\}.
\end{equation}
From Lemma~\ref{lem:beta_2} we know that conditioned on $X$, $\bbP (\cE_1)\geq 1-\delta$ over the randomness in $\Omega$.  Using equation~\eqref{eqn:this16} we get that on event $\cE_1$
\begin{equation}
\label{eqn:zhat_inf}
|\Zhat|_{\infty}=O\left(\sigma_{r+1}(X)+\frac{mn\log(\frac{m+n}{\delta})}{\siom}+\sqrt{\frac{mn\log((m+n)/\delta)}{3\siom}}\right)\defeq b
\end{equation}
Now let $I_4'$ be the term argument to the expectation operator in $I_4$.
Now let us define another event 
\begin{equation}
\cE_{11}=\left\{I_4'\leq \sqrt{\frac{b\log((m+n)/\delta)}{\siom}}\right\}
\end{equation}
 Using lemma~\eqref{lem:large_dev}, we get that $\bbP (\cE_{11})\geq 1-\delta$ over random choice of $\Omega$. Notice that $I_4'\leq 4$ surely. We are now ready to calculate $I_4$ as follows
 \begin{align}
 I_4&=\bbE_X \bbE_{\Omega|X} I_4'\\
&\leq\bbE_X \bbP(\cE_1)\bbE_{\Omega|X,\cE_1}I_4'+4\bbP(\bar{\cE_1})\\
&\leq\bbE_X \bbP(\cE_1)(\bbP(\cE_{11})I_4'+4\bbP(\bar{\cE}_{11}))+4\bbP(\bar{\cE_1})\\
&\leq 8\delta+\bbE_X\sqrt{\frac{b\log((m+n)/\delta)}{\siom}}
 \end{align}
Substituting the value of $b$, and using $\delta=\frac{1}{\siom}$, and using the assumption that $\sigma_{r+1}(X)\leq O(\sqrt{n})$ with high probability, we get that
\begin{align}
I_4&= \bbE_{X}\bbE_{\Omega|X} I_4'\\
&\leq 8\delta+\bbE_X \sqrt{\frac{1}{\siom}O\left(\sigma_{r+1}(X)+\frac{mn\log((m+n)\siom)}{\siom}+\sqrt{\frac{mn\log((m+n)\siom)}{3\siom}}\right)}\\
&\leq O\left(\sqrt{\frac{mn}{\siom^2}\log^2\left((m+n)\siom\right)}\right)
\end{align}

Notice that $\Delta_1$ uses $\ghat-\gstar$ which is a 2 Lipchitz function. By perfoming a similar analysis as in $I_4$ it is easy to show that $\Delta_1=O(I_4)$.

\textbf{Bounding $I_1$}.
\begin{align}
I_1&=\bbE\left[\frac{1}{\siom}\sum_{\Omega}(\Xij-\gstar(\Zhatij))^2- \frac{1}{\siom}\sum_{\Omega}(\Xij-\gstar(\Zstarij))^2\right]\\
&=\bbE\left[\frac{1}{\siom}\sum_{\Omega} (\gstar(\Zstarij)-\gstar(\Zhatij))(2\Xij-\gstar(\Zhatij)-\gstar(\Zstarij))\right]\\
&\leqa 4\bbE \frac{1}{\siom}|\gstar(\Zstarij)-\gstar(\Zhatij)\|\\
&\leqb 4\bbE \frac{1}{\siom}\sum_{\Omega}|\Zstarij-\Zhatij|\\
&= 4\bbE \frac{1}{mn}\sum_{i,j}|\Zstarij-\Zhatij|+4\underbrace{\left(\bbE \frac{1}{\siom}\sum_{\Omega}|\Zstarij-\Zhatij|-\bbE \frac{1}{mn}\sum_{i,j}|\Zstarij-\Zhatij|\right)}_{\Delta_5}\\
&\leqc 4\bbE \frac{1}{mn}\sum_{i,j}|\Zstarij-\Zhatij|+4\Delta_5\\
&\leqd 4\sqrt{\bbE \frac{1}{mn}\sum_{i,j}|\Zstarij-\Zhatij|^2}+4\Delta_5
=4(\sqrt{\epsilon_2}+\Delta_5)
\end{align}
where, to get inequality (a) we used the fact that $|\Xij|\leq 1$ and $|\gstar|\leq 1$. To get inequality (b) we used the fact that $\gstar$ is $1$ Lipschitz. To get inequality (c) we used concentration of measure. Finally, to get inequality (d) we used Jensen's inequality to bound $\bbE |x|\leq \sqrt{E x^2}$.
Our next step is to bound $\Delta_5$.

\textbf{Bounding $\Delta_5$:}
The idea is to consider the event $\cE_1$ as done during bounding the term $I_4$. Once again we shall consider the event
\begin{equation}
\cE_1=\left\{||A-X||\leq \frac{2mn\log\left(\frac{m+n}{\delta}\right)}{3\siom}+\sqrt{\frac{2\log(\frac{m+n}{\delta})mn}{\siom}}\right\}.
\end{equation}
Similar to arguments there, we know from Equation~\eqref{eqn:zhat_inf} that on event $\cE_1$
\begin{equation*}
|\Zhat|_{\infty}=O\left(\sigma_{r+1}(X)+\frac{mn\log(\frac{m+n}{\siom})}{\siom}+\sqrt{\frac{mn\log((m+n)/\delta)}{3\siom}}\right)\defeq b
\end{equation*}
Consider the collection of random variables $\xi_1,\ldots,\xi_{\siom}$, where each $\xi_k$ takes the value $\Zstarij-\Zhatij$, where $(i,j)$ is chosen u.a.r. with replacement from $[n]\times [m]$. It is easy to see that each of $\xi_k\in[0,b+1]$ on $\cE_1$. Applying Hoeffding inequality we get  on $\cE_1$ with probability at least $1-\delta$ over the random choice of $\Omega$, and on event $\cE_1$
\begin{equation}
\frac{1}{\siom}\sum_{\Omega}|\Zstarij-\Zhatij|-\sum_{i,j}|\Zstarij-\Zhatij|\leq \sqrt{\frac{(b+1)^2}{2\siom}\log(1/\delta)}
\end{equation}
By arguments similar to the ones used in establishing bounds for $I_4$, we get
\begin{align}
\Delta_5\leq O\left(\log\left((m+n)\siom\right)\left(\sqrt{\frac{n}{\siom}}+\frac{mn}{\siom^{3/2}}+\frac{\sqrt{mn}}{\siom}\right)\right).
\end{align}
This concludes our first set of calculations. With this we have
\begin{equation}
\MSE(\Mhat)=O\left(\frac{\sqrt{mn\log(n)}}{\siom}+\sqrt{\frac{n}{\siom}}+\frac{mn}{\siom^{3/2}}+\frac{\sqrt{mn}}{\siom}+\epsilon_2+\sqrt{\epsilon_2}\right)
\end{equation}
\end{proof}
The rest of the proof establishes upper bounds on $\epsilon_2$.

\subsection{Bounding $\epsilon_2$.} In order to establish upper bound on $\epsilon_2$ we first need the following projection lemma. This lemma is similar in spirit to a lemma of S.Chatterjee~\citep[Lemma 3.5]{chatterjee2014matrix}. Before we establish this lemma, we would like to clarify the notation that we use. Given a matrix $A\in \bbR^{n\times m}$,  with $m\leq n$, denote by $\sigma_1(A)\geq \sigma_2(A)\geq \ldots \geq \sigma_m(A)$ the singular values of $A$ in decreasing order.
\begin{lemma}
\label{lem:proj}
Let $A=\sum_{i=1}^m \sigma_i x_iy_i^\top$ be the SVD of a rectangular matrix $A\in \bbR^{m\times n}$, with the singular values $\sigma_1\geq \sigma_2\ldots \geq \sigma_m$ arranged in decreasing order. Let $B$ be an unknown $m\times n$ matrix. Given $1\leq r\leq m$, let $\hat{B}\defeq P_{r}(A)\defeq \sum_{i=1}^r \sigma_i x_i y_i^\top$ be the projection estimator of $B$. Then,
\begin{equation}
||P_r(A)-B||_F\leq \sqrt{||B||_{\star} (\sigma_{r+1}+||A-B||)}+2\sqrt{2r}(\sigma_{r+1}+||A-B||).
\end{equation}
\end{lemma}
\begin{proof}
Let $B=\sum_{i=1}^m \tau_i u_iv_i^\top$ be the SVD of $B$ with $\tau_1\geq \tau_2\geq \ldots \tau_m$.
\begin{equation}
\label{eqn:this1}
||\hat{B}-B||_F\leq ||\Bhat-G||_F+||G-B||_F,
\end{equation}
and 
\begin{equation}
\label{eqn:this2}
||G-B||_F^2 = ||\sum_{i\geq r+1} \tau_i u_iv_i^\top||_F^2=\sum_{i\geq r+1} \tau_i^2\leq (\max_{i\geq r+1}\tau_i)||B||_\star.
\end{equation}
Let $\delta_1\geq \delta_2\geq \ldots$ be the singular values of $A-B$ in decreasing order. Then from Weyl's inequality we know that
\begin{equation}
\max_i |\sigma_i-\tau_i|\leq \max_i \delta_i=||A-B||.
\end{equation}
Hence, for $i\geq r+1$, 
\begin{equation}
\label{eqn:this3}
\tau_i\leq \sigma_i+||A-B||\leq \sigma_{r+1}(A)+||A-B||.
\end{equation}
This allows us to conclude that $\max_{i\geq r+1} \tau_i\leq \sigma_{r+1}(A)+||A-B||$. Combined with Equation~\eqref{eqn:this2} we get
\begin{equation}
\label{eqn:this3.5}
||G-B||_F^2\leq ||B||_\star(\sigma_{r+1}(A)+||A-B||).
\end{equation}
Next, we shall upper bound the quantity $||\Bhat-G||_F$. By construction, both $\Bhat$ and $G$ are rank $r$ matrices and hence $\Bhat-G$ is a rank $2r$ matrix. This allows us to control the Frobenius norm of $\Bhat-G$ via its spectral norm as follows
\begin{equation}
\label{eqn:this4}
||\Bhat-G||_F\leq \sqrt{2r}||\Bhat-G||
\end{equation}
To bound $||\Bhat-G||$ consider the following decomposition
\begin{equation}
\label{eqn:this5}
||\Bhat-G||\leq ||\Bhat-A||+||A-B||+||B-G||.
\end{equation}
We have
\begin{equation}
\label{eqn:this6}
||\Bhat-A||=||\sum_i \sigma_i x_iy_i^\top||\leq \sigma_{r+1}.
\end{equation}
\begin{equation}
\label{eqn:this7}
||B-G||=||\sum_{i\geq r+1} \tau_i u_i v_i^\top||=\tau_{r+1}\leqa \sigma_{r+1}+||A-B||
\end{equation}
where to get inequality (a) we used Equation~\eqref{eqn:this3}. Combining Equations~\eqref{eqn:this5},~\eqref{eqn:this6},~\eqref{eqn:this7} we get
\begin{equation}
||\Bhat-G||\leq \sigma_{r+1}+||A-B||+\sigma_{r+1}+||A-B||=2(\sigma_{r+1}+||A-B||)
\end{equation}
and using Equation~\eqref{eqn:this4} we get 
\begin{equation}
\label{eqn:this8}
||\Bhat-G||_F\leq 2\sqrt{2r}(\sigma_{r+1}+||A-B||)
\end{equation}
Finally using Equation~\eqref{eqn:this3.5} and Equation~\eqref{eqn:this8} we get 
\begin{equation}
||\Bhat-B||_F\leq 2\sqrt{2r}(\sigma_{r+1}+||A-B||)+\sqrt{||B||_\star(\sigma_{r+1}+||A-B||)}.
\end{equation}
\end{proof}
In order to obtain an upper bound on $\epsilon_2$ we shall use Lemma~\eqref{lem:proj} with the following choices for matrices $A,B$:
$
A\defeq\frac{1}{p\siom}X\circ \Deltaom, B\defeq\Zstar, \Zhat\defeq P_r(A).
$. We then get 
\begin{equation}
||\Zhat-\Zstar||_F\leq \sqrt{||\Zstar||_{\star} (\sigma_{r+1}+||A-\Zstar||)}+2\sqrt{2r}(\sigma_{r+1}+||A-\Zstar||).
\end{equation}
Since $\Zstar$ is of rank $r$, we have $||\Zstar||_\star\leq r||\Zstar||$. From triangle inequality $||A||\leq ||A-\Zstar||+||\Zstar||$. These facts coupled with the fact that $\sigma_{r+1}\leq \sigma_1$ allows us to obtain 
\begin{equation}
\label{eqn:this9}
\bbE ||\Zhat-\Zstar||_F^2 \leq r||\Zstar||(2\bbE||A-\Zstar||+||\Zstar||)+8r(4\bbE||A-\Zstar||^2+||\Zstar||^2).
\end{equation}
Notice that $\epsilon_2$ is a scaled version of $\bbE ||\Zhat-\Zstar||_F^2$. Let,
\begin{align}
\beta_1&\defeq \bbE||A-X||\\
\beta_2&\defeq \bbE||A-X||^2.
\end{align}
Using the above definitions, Equation~\eqref{eqn:this9}, the triangle inequality $||A-\Zstar||\leq ||A-X||+||X-\Zstar||$, along with the elementary fact that $(a+b)^2\leq 2a^2+2b^2$, we obtain
\begin{align}
\label{eqn:pivot}
\bbE ||\Zhat-\Zstar||_F^2 &\leq r||\Zstar||(2\beta_1+2\bbE||X-\Zstar||+||\Zstar||)+8r(8\beta_2+8\bbE||X-\Zstar||^2+||\Zstar||^2)\\
&=r||\Zstar||(2\bbE ||X-\Zstar||+||\Zstar||)+8r(8\bbE ||X-\Zstar||^2+||\Zstar||^2)+r(2\beta_1+64\beta_2).
\end{align}
\textbf{Bounding $\beta_1,\beta_2$}. 
In order to bound $\beta_1,\beta_2$ we need upper bounds on spectral norm of sums of random matrices. Towards this, the following Bernstein inequality is useful
\begin{thm}[Bernstein's inequality]  
\label{thm:bernstein}
Let $S_1,\ldots S_k$ be independent, centered random matrices with common dimension $n\times m$, and assume that each one of them is bounded
\begin{equation}
||S_j||\leq L \text{~for each~} j\geq 1.
\end{equation}
Let $M=\sum_{j=1}^k S_j$, and let $\nu(M)$ denote the matrix variance statistic of the sum 
\begin{equation}
\nu(M)=\max\Bigl\{||\sum_{j=1}^k \bbE S_jS_j^\top||,||\sum_{j=1}^k \bbE S_j^\top S_j||\Bigr\}. 
\end{equation}
Then
\begin{enumerate}
\item 
\begin{equation}
\bbP(||M||\geq t)\leq (m+n)\exp\left(\frac{-t^2/2}{\nu(M)+Lt/3}\right),
\end{equation}
Furthermore 
\item \begin{equation}
\bbE Z\leq \sqrt{2\nu(M)\log(m+n)}+\frac{1}{3}L\log(m+n).
\end{equation}
\end{enumerate}
\end{thm} 
We shall bound $\beta_1$ using part (ii) of Bernstein's inequality, and $\beta_2$ using part (ii) of Bernstein's inequality.  The next two lemma's provide necessary material for bounding $\beta_1,\beta_2$.
\begin{proposition}
\label{prop:L}
Let $\Delta$ be a random mask of size $n\times m$, where a random location is chosen and set to $1$, and rest of the entries are set to 0. Let $X$ be a matrix of size $n\times m$ with entries bounded in absolute value by 1. Define $S=\frac{1}{p}X\circ \Delta-X$. Let $p=\frac{1}{mn}$. Then,
\begin{enumerate}
 \item $||S||\leq ||X||+\frac{1}{p}$
 \item $\bbE S^\top S=\bbE SS^\top=X\circ X-XX^\top$
\end{enumerate}
\end{proposition}
\begin{proof}
\begin{equation}
||S||=||\frac{1}{p}X\circ \Delta-X||\leq ||X||+||\frac{1}{p}X\circ\Delta||\leqa ||X||+\frac{1}{p}.
\end{equation}
In the above set of inequalities in order to derive (a) we used the fact that $X\circ \Delta$ is an $n\times m$ matrix with a single non-zero entry bounded in absolute value by 1. Hence the spectral norm of this matrix will be bounded by $1$. To derive the second part of the proposition we proceed as follows
\begin{align}
\bbE SS^\top=\bbE (\frac{1}{p}X\circ \Delta-X)(\frac{1}{p}X\circ \Delta-X)^\top=\bbE [\frac{1}{p^2}(X\circ \Delta) (X\circ \Delta)^\top-\frac{1}{p}(X\circ \Delta) X^\top-\frac{1}{p}X(X\circ\Delta)^\top+XX^\top].
\end{align}
Via elementary calculations, it is easy to verify that
\begin{align}
\bbE \left[\frac{1}{p^2} (X\circ \Delta)(X\circ \Delta)^\top\right]=X\circ X\\
\bbE \left[\frac{1}{p}X(X\circ\Delta)^\top\right]=\bbE \left[\frac{1}{p}(X\circ\Delta)X^\top\right]=XX^\top.
\end{align}
These identities allow us to conclude part (ii) of this proposition. 
\end{proof}
We are now ready to bound the quantities $\beta_1,\beta_2$
\begin{lemma}
Let $p=\frac{1}{mn}$. Then,
\label{lem:beta_1}
\begin{equation}
\beta_1=\bbE\left\|\frac{1}{p\siom}X\circ \Deltaom-X\right\|\leq \sqrt{\frac{2\log(m+n)||X\circ X-XX^\top||}{|\Omega|}}+\frac{\log(m+n)(p||X||+1)}{3p\siom}.
\end{equation}
\end{lemma}
\begin{proof}
\begin{align}
\left\|\frac{1}{p\siom}X\circ \Deltaom-X\right\|=\frac{1}{\siom}\left\|\sum_{j=1}^{\siom} \underbrace{(X\circ \Delta_j-X)}_{S_j}\right\|
\end{align}
Here $\Delta_1,\ldots,\Deltaom$ are random i.i.d. boolean masks with each of them having exactly one non-zero, whose location is chosen uniformly at random from $[n]\times [m]$. For this reason the matrices $S_1,\ldots,S_\siom$ are i.i.d. matrices. It is easy to see that $\bbE S_j=0$ for each $j\geq 1$. Applying Bernstein's inequality (Theorem~\eqref{thm:bernstein}) and using Proposition~\eqref{prop:L} to bound the necessary quantities we get that 
\begin{align}
\bbE \left\|\frac{1}{p\siom}X\circ \Deltaom-X\right\|&=\frac{1}{\siom}\left[\sqrt{2\log(m+n)\siom||X\circ X-XX^\top||}+\frac{\log(m+n)}{3}(||X||+\frac{1}{p})\right]\\
&= \sqrt{\frac{2\log(m+n)||X\circ X-XX^\top||}{|\Omega|}}+\frac{\log(m+n)(p||X||+1)}{3p\siom}
\end{align}
\end{proof}
Next we bound $\beta_2$.
\begin{lemma}
\label{lem:beta_2}
Let $p=\frac{1}{mn}$. Then,
\begin{equation}
\beta_2=\bbE\left\|\frac{1}{p\siom}X\circ \Deltaom-X\right\|^2\leq 1+\left(\frac{20mn\log(n)}{3\siom}\right)^2+\frac{10\log(n)}{\siom}||X\circ X-XX^\top||.
\end{equation}
\end{lemma}
\begin{proof}
Using part (i) of Bernstein's inequality we get that, for any $\delta>0$, with probability at least $1-\delta$, 
\begin{equation}
\label{eqn:this12}
\left\|A-X\right\|\leq \frac{2\log\left(\frac{m+n}{\delta}\right)}{3\siom}\left(||X||+\frac{1}{p}\right)+\sqrt{\frac{2\log(\frac{m+n}{\delta})||X\circ X-XX^\top||}{\siom}}.
\end{equation}
Worst case upper bound on $\|A-X\|$ can be derived as follows
\begin{align}
\label{eqn:this13}
\|A-X\|&=||\frac{1}{p\siom}X\circ \Deltaom-X||\\
&\leq \frac{1}{p\siom}||X\circ \Deltaom||+||X||\\
&\leq \frac{1}{p\siom}\sum_{j=1}^{\siom}||X\circ \Delta_{j}||+||X||\\
&\leq \frac{1}{p}+||X||.
\end{align}
Using equations~\eqref{eqn:this12} and ~\eqref{eqn:this13} we get
\begin{equation}
\bbE||A-X||^2\leq (1-\delta)\left(\frac{2\log\left(\frac{m+n}{\delta}\right)}{3\siom}\left(||X||+\frac{1}{p}\right)+\sqrt{\frac{2\log(\frac{m+n}{\delta})||X\circ X-XX^\top||}{\siom}}\right)^2+\delta\left(\frac{1}{p}+||X||\right)^2
\end{equation}
Since each element of $X$ is bounded by $1$ in magnitude, we get that $||X||\leq \sqrt{mn}$. Now, replace $p=\frac{1}{mn}$ and choose $\delta=\frac{1}{\left(mn+\sqrt{mn}\right)^2}$. Using the inequality $(a+b)^2< 2a^2+2b^2$ and over-approximating we get the desired result.
\end{proof}

\textbf{Final bound on $\epsilon_2$.} We are now ready to establish a bound on $\epsilon_2$. In the next bound we shall no longer keep track of explicit constants. Instead in the following calculations we shall use a universal constant $C>0$ whose value can change from one line to another.
\begin{lemma}
\label{lem:eps_2}
Let $\mu_1=\bbE ||N||,\mu_2=\bbE ||N||^2$. Then, for some universal constant $C>0$ we have
\begin{equation}
\begin{split}
\epsilon_2\leq O\left(\frac{r}{m\sqrt{n}}(||\Mstar-\Zstar||+\mu_1)+\frac{r||\Mstar-\Zstar||^2}{mn}+\frac{r\mu_2}{mn}+\frac{r}{m}+\frac{rmn\log^2(n)}{\siom^2}\right)
\end{split}
\end{equation}
\end{lemma}
\begin{proof}
From Equation~\eqref{eqn:pivot} we have
\begin{align}
\epsilon_2\leq r||\Zstar||(2\bbE ||X-\Zstar||+||\Zstar||)+8r(8\bbE ||X-\Zstar||^2+||\Zstar||^2)+r(2\beta_1+64\beta_2). 
\end{align}
Now, using Lemma~\eqref{lem:beta_1} and ~\eqref{lem:beta_2} to bound $\beta_1,\beta_2$, we get 
\begin{equation}
\label{eqn:final_eps2}
\begin{split}
\epsilon_2\leq \frac{Cr}{mn}\bbE\Bigl[||\Zstar||~||X-\Zstar||+ ||X-\Zstar||^2+||\Zstar||^2+ \sqrt{\frac{\log(n)||X\circ X-XX^\top||}{\siom}}+\\
\frac{\log(n)}{3\siom}(||X||+mn)+1+\frac{m^2n^2\log^2(n)}{\siom^2}+\frac{\log(n)}{\siom}||X\circ X-XX^\top||\Bigr].
\end{split}
\end{equation}
In the above expectation the expectation is being taken w.r.t. the randomness in $X$ due to additive noise of our model. We shall now compute the remaining expectations. For notational convenience, define $\mu_1=\bbE ||N||$, and $\mu_2=\bbE ||N||^2$. Using the fact that $X=\Mstar+N$, we get
\begin{align}
\label{eqn:hadam}
\bbE \|X\circ X-XX^\top\|&\leq \bbE ||X\circ X\|+\bbE ||XX^\top||\\
&\leqa \bbE ||X||^2+\bbE ||(\Mstar+N)(\Mstar+N)^\top||\\
&\leqb \bbE [||\Mstar||^2+||N||^2+2||\Mstar|| ~||N||+ ||\Mstar(\Mstar)^\top||+ \Mstar N^\top+ N(\Mstar)^\top+ NN^\top]\\
&=  2||\Mstar||^2+2\mu_2+4||\Mstar||\mu_1
\end{align}
where to obtain inequality (a) we used the fact that $||A\circ B||\leq ||A|| ||B||$~\citep[Problem 1.6.13, page 23]{bhatia1997matrix}. To obtain inequality (b) we used sub-additivity of norms, and the fact that spectral norm is sub-multiplicative.
By Jensen's inequality we get 
\begin{equation}
\label{eqn:sqrt_hadam}
\bbE \sqrt{\|X\circ X-XX^\top\|}\leq \sqrt{\bbE \|X\circ X-XX^\top\|}\leq \sqrt{2||\Mstar||^2+2\mu_2+4\mu_1||\Mstar||}
\end{equation}
Finally using the sub-additivity of norms we get that 
\begin{align}
\label{eqn:this14}
\bbE||X-\Zstar||^2&=\bbE||\Mstar+N-\Zstar||^2\leq 2||\Mstar-\Zstar||^2+2\bbE ||N||^2=2||\Mstar-\Zstar||^2+2\mu_2\\
\bbE||X-\Zstar||& =\bbE||\Mstar+N-\Zstar||\leq \bbE ||\Mstar-\Zstar||+\bbE ||N||=\bbE ||\Mstar-\Zstar||+\mu_1
\end{align}
Now, putting together Equations~\eqref{eqn:final_eps2},~\eqref{eqn:hadam},~\eqref{eqn:sqrt_hadam},~\eqref{eqn:this14}, and substituting the worst case bound $||\Mstar||=C\sqrt{mn}$, we get
\begin{equation}
\begin{split}
\epsilon_2&\leq C\Bigl[\frac{r}{mn}||\Zstar||\left(||\Mstar-\Zstar||+\mu_1\right)+\frac{r}{mn}||\Mstar-\Zstar||^2+\frac{r}{mn}(\mu_2+||\Zstar||^2)+
\\
&\quad \frac{r}{mn}\sqrt{\frac{\log(n)}{\siom}\left(mn+\mu_1\sqrt{mn}+\mu_2\right)}+\frac{rmn\log^2(n)}{\siom^2}+\frac{r\log(n)}{mn\siom}(mn+\mu_2+\mu_1\sqrt{mn})\Bigr].
\end{split}
\end{equation}
We can further simplify the above expression, by noting that the entries of $N$  are bounded by $1$, and hence $\mu_1=O(\sqrt{mn}),\mu_2=O(mn)$. Note that in reality $\mu_1,\mu_2$ are much smaller, and one could lose a lot of information by considering their worst case values. However, in order to simplify the above bound for $\epsilon_2$ and make it interpretable, we shall selectively replace $\mu_1,\mu_2$ by $\sqrt{mn},mn$ respectively, This allows us to gauge which terms are lower order terms and drop them. This gets us
\begin{equation}
\epsilon_2\leq O\left(\frac{r}{m\sqrt{n}}(||\Mstar-\Zstar||+\mu_1)+\frac{r||\Mstar-\Zstar||^2}{mn}+\frac{r\mu_2}{mn}+\frac{r}{m}+\frac{rmn\log^2(n)}{\siom^2}\right)
\end{equation}
\end{proof}
\section{Proof of Theorem~\eqref{thm:main}}
From Lemma~\eqref{lem:eps_delta} we have 
\begin{equation*} 
\MSE(\Mhat)\leq O\left(\frac{\sqrt{mn\log(n)}}{\siom}+\sqrt{\frac{n}{\siom}}+\frac{mn}{\siom^{3/2}}+\frac{\sqrt{mn}}{\siom}+\epsilon_2+\sqrt{\epsilon_2}\right)
\end{equation*}
 From Lemma~\eqref{lem:eps_2} we have 
 \begin{equation}
\begin{split}
\epsilon_2\leq O\left(\frac{r}{m\sqrt{n}}(||\Mstar-\Zstar||+\mu_1)+\frac{r||\Mstar-\Zstar||^2}{mn}+\frac{r\mu_2}{mn}+\frac{r}{m}+\frac{rmn\log^2(n)}{\siom^2}\right)
\end{split}
\end{equation}
Putting the above two equations together we get 
\begin{equation}
\begin{split}
\MSE(\Mhat)&=O\Bigl(\sqrt{\frac{r}{m}}+\frac{\sqrt{mn\log(n)}}{\siom}+\frac{mn}{\siom^{3/2}}+\frac{\sqrt{mn}}{\siom}+\sqrt{\frac{r}{m\sqrt{n}}\left(\mu_1+\frac{\mu_2}{\sqrt{n}}\right)}+\\
&\qquad \sqrt{\frac{r\alpha}{m\sqrt{n}}\left(1+\frac{\alpha}{\sqrt{n}}\right)}+
 \sqrt{\frac{rmn\log^2(n)}{\siom^2}}\Bigr)
\end{split}
\end{equation}
\section {Source for datasets}
Here is where one can download the real world datasets on which all of our experiments were performed.
\begin{enumerate}
\item Paper recommendation dataset:\url{http://www.comp.nus.edu.sg/~sugiyama/SchPaperRecData.html}.
\item Jester dataset: \url{http://goldberg.berkeley.edu/jester-data/}.
\item Movie lens dataset: ~\url{http://grouplens.org/datasets/movielens/}
\item Cameraman dataset:~\url{http://www.utdallas.edu/~cxc123730/mh_bcs_spl.html}
\end{enumerate}
\section{RMSE plots with iterations}
In Figure~\eqref{fig:suppl} we show how the RMSE of $\mmc-$c algorithm changes with iterations. These plots were made on the synthetic datasets that were used in our experiments. The value of $p$ was set to $0.35$. As one can see, on an average, there is a decreasing trend in the RMSE. This decrease is almost linear for small values of $c$ and sub-linear for larger values of $c$.
\begin{figure}[h!]
\centering
    \includegraphics[height=3.0in]{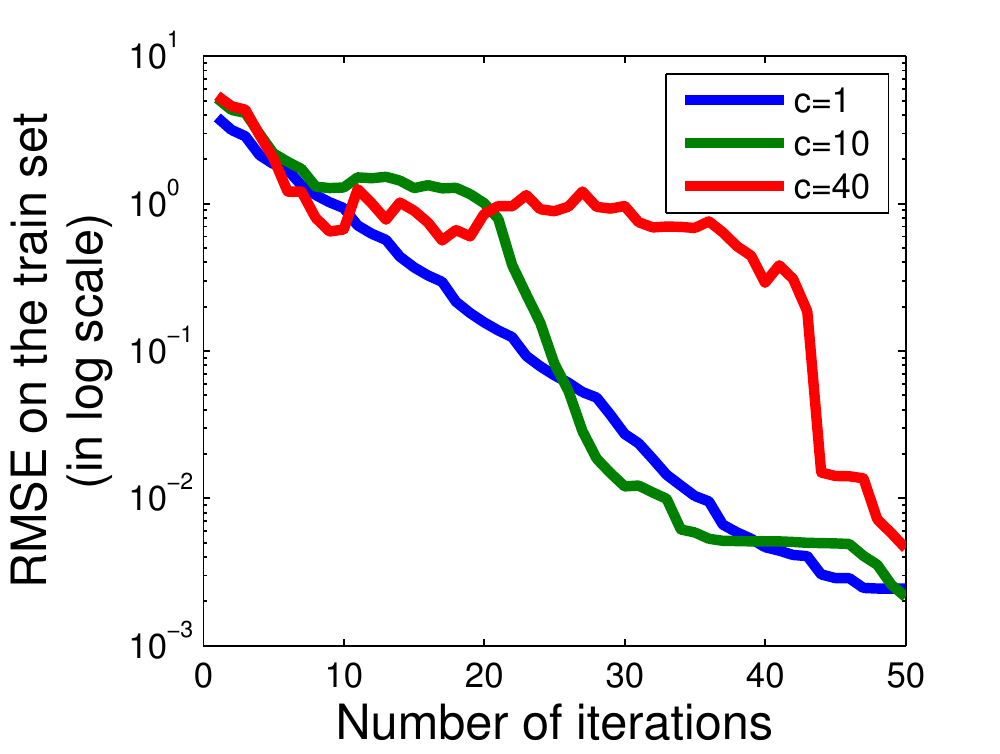}
\caption{RMSE of MMC-c with increasing iterations. \label{fig:suppl}}
\end{figure}

\end{document}